\documentclass{article}

\usepackage{microtype}
\usepackage{graphicx}
\usepackage{booktabs} 

\usepackage[accepted]{icml2020}

\usepackage[utf8]{inputenc} 
\usepackage[T1]{fontenc}    
\PassOptionsToPackage{hyphens}{url}\usepackage{hyperref}
\usepackage{booktabs}       
\usepackage{amsfonts}       
\usepackage{nicefrac}       
\usepackage{microtype}      

\usepackage{subcaption}
\usepackage{multirow}
\usepackage{xcolor}\usepackage{xfrac}
\usepackage{amssymb}
\usepackage[most]{tcolorbox}

\usepackage{amsthm}
\newtheorem{theorem}{Theorem}
\newtheorem{definition}{Definition}

\newtheorem{assumption}{Assumption}

\newtheorem{proposition}{Proposition}

\usepackage{amsmath,amsfonts,bm}

\def\Figref#1{Figure~\ref{#1}}

\def\Secref#1{Section~\ref{#1}}

\def\eqref#1{equation~\ref{#1}}
\def\Eqref#1{Equation~\ref{#1}}

\def\1{\bm{1}}

\def\rvu{{\mathbf{i}}}

\def\rvo{{\mathbf{o}}}

\def\rvs{{\mathbf{s}}}

\def\rvu{{\mathbf{u}}}

\def\rvx{{\mathbf{x}}}
\def\rvy{{\mathbf{y}}}

\def\ve{{\bm{e}}}

\def\vo{{\bm{o}}}

\def\vs{{\bm{s}}}

\def\vu{{\bm{u}}}

\DeclareMathAlphabet{\mathsfit}{\encodingdefault}{\sfdefault}{m}{sl}
\SetMathAlphabet{\mathsfit}{bold}{\encodingdefault}{\sfdefault}{bx}{n}

\usepackage{algorithm}

\definecolor{lightgrey}{rgb}{0.43,0.43,0.43}

\icmltitlerunning{Causal World Models by Unsupervised Deconfounding of Physical Dynamics}

\begin{document}

\twocolumn[
\icmltitle{Causal World Models by Unsupervised Deconfounding of Physical Dynamics}



\icmlsetsymbol{equal}{*}

\begin{icmlauthorlist}
\icmlauthor{Minne Li}{equal,ucl}
\icmlauthor{Mengyue Yang}{equal,ucl}
\icmlauthor{Furui Liu}{huawei}
\icmlauthor{Xu Chen}{ucl}
\icmlauthor{Zhitang Chen}{huawei}
\icmlauthor{Jun Wang}{ucl}
\end{icmlauthorlist}
\author{%
  Minne Li\thanks{Equal contributions.}~~$^\ddagger$, Mengyue Yang\samethanks~~$^\dagger$, Furui Liu$^\dagger$, Xu Chen$^\ddagger$, Zhitang Chen$^\dagger$, and Jun Wang$^\ddagger$\\
  $^\ddagger$University College London,  $^\dagger$Huawei Noah’s Ark Lab
}

\icmlaffiliation{ucl}{University College London, London, United Kingdom}
\icmlaffiliation{huawei}{Huawei Noah’s Ark Lab, Shenzhen, China}

\icmlcorrespondingauthor{Furui Liu}{liufurui2@huawei.com}

\icmlkeywords{Machine Learning, ICML}

\vskip 0.3in
]



\printAffiliationsAndNotice{\icmlEqualContribution} 

\begin{abstract}
The capability of imagining internally with a mental model of the world is vitally important for human cognition. If a machine intelligent agent can learn a world model to create a "dream" environment, it can then internally ask what-if questions -- simulate the alternative futures that haven't been experienced in the past yet -- and make optimal decisions accordingly. Existing world models are established typically by learning spatio-temporal regularities embedded from the past sensory signal without taking into account confounding factors that influence state transition dynamics. As such, they fail to answer the critical counterfactual questions about "what would have happened" if a certain action policy was taken. In this paper, we propose Causal World Models (CWMs) that allow unsupervised modeling of relationships between the intervened observations and the alternative futures by learning an estimator of the latent confounding factors. We empirically evaluate our method and demonstrate its effectiveness in a variety of physical reasoning environments. Specifically, we show reductions in sample complexity for reinforcement learning tasks and improvements in counterfactual physical reasoning.
\end{abstract}

\section{Introduction}
\label{sec:intro}
Human-level intelligence relies on building up the capability of simulating the physical world in order to create human-like thinking, reasoning, and decision making abilities~\cite{Lake1332,lake2017building,spelke2007core}.
This mechanism has served as a core motivation behind several recent works of learning world models (WMs) that aim at predicting future sensory data given the agent's current motor actions~\cite{greff2017neural,Kipf2020Contrastive,van2018relational,xu2019unsupervised}.
One of the ultimate goals of WMs is creating the \emph{dream},
where the agent can internally simulate the alternative futures not encountered in the real world~\cite{ha2018worldmodels}.
This, however, requires WMs to predict the reasonable outcome of various types of interventions to the observation.
In this paper,
we explore building WMs capable of creating the dream environment to predict the counterfacts that \emph{would have happened}.

Conventional world models usually aim at learning spatio-temporal regularities from the past data and thereby predict future frames from one or several past frame(s).
These models typically employ standard IID\footnote{Independent and identically distributed.} function learning techniques without considering the causal effect of the interventions~\cite{DBLP:conf/uai/BalkeP95}.
However,
as an agent interacts with the environment and thus influences the statistics of the gathered data over time,
the IID assumption is violated~\cite{schlkopf2019causality}.
By looking for invariances from observational data has been shown to help identify robust components and causal features of the environment~\cite{peters2017elements}.
For example in the physical dynamics systems,
these features could be the underneath confounding factors (or confounders) affecting the environment across temporality such as inertia, velocity, gravity or friction.
Accurate estimation of these factors enables model-based RL agents to generalize to other parts of the state space and thus being more robust.
Correcting for the confounding effect could also influence the policy to be optimized~\cite{lu2018deconfounding}.

In this paper, we introduce the Causal World Models (CWMs) to predict the effect of interventions that \emph{would have happened},
i.e.,
we wish to predict consequent observation trajectories if we had changed the initial observation by performing an external intervention,
which is defined as an object-level observable change applied to the initial observation.
CWMs aim at modeling relationships between the variable on which the intervention is performed and the variable whose alternative future should be predicted.
Our model starts from learning a set of abstract state variables for each object in a particular observation and model the transition using graph neural networks~\cite{DBLP:journals/corr/abs-1806-01261,kipf2017semi,DBLP:journals/corr/LiTBZ15,10.1109/TNN.2008.2005605} that operate on latent abstract representations.
We then develop CWMs by estimating the latent representation of the confounders,
i.e., a set of static but visually unobservable quantities that could otherwise affect the observation.
In our method, the deconfounding process could be understood as an extension on conventional World Models.
Deconfounding methods help to build time-invariant function beyond transition functions, which gives a more precise description of the environment by preserving global information,
such as gravity acceleration in physical systems,
from time variations.
The alternative future is then predicted given the altered past and the estimation of confounders.
To further reduce the intrinsic bias owing to the inadequate counterfactual records,
we also propose a counterfactual risk minimization method from historical observations for CWMs.
We empirically evaluate our method and demonstrate its effectiveness in CoPhy~\cite{Baradel2020CoPhy} (a counterfactual physics benchmark suite) and PHYRE~\cite{bakhtin2019phyre} (a physical reasoning benchmark with reinforcement learning tasks).
Specifically, we show reductions in sample complexity for reinforcement learning and improvements in counterfactual physical reasoning.

\section{Related Work}
The task in observational discovery of causal effects in physical reality is usually concerned with predicting the effect of various types of interventions,
including physical laws underneath the environment~\cite{NIPS2016_6418,DBLP:conf/iclr/0003U0T17,DBLP:conf/nips/WuYLFT15},
the actions executed by the agent itself~\cite{levine2014learning,DBLP:conf/nips/LiW0B19,wahlstrom2015pixels,watter2015embed},
and the outcome of other agents' decision in multi-agent systems~\cite{pmlr-v48-he16,DBLP:conf/ijcai/TianWGPZW19}.
While causal reasoning has gained mainstream attention in the machine learning field recently~\cite{Lopez-Paz_2017_ICLR,mooij2016distinguishing,schlkopf2019causality},
the current literature mostly focuses on discovering the causal effect between variables in static environments~\cite{10.5555/3020847.3020867,kocaoglu2018causalgan,lopez2017discovering}.
Extensive research has also been conducted on visual reasoning in static environments~\cite{10.5555/3157382.3157459,arad2018compositional,johnson2017inferring,mao2018the,santoro2017simple}.
Although~\citet{lu2018deconfounding} has considered deconfounding in the reinforcement learning scenario,
their assumption only covers the confounding between the observations, actions and rewards,
thus failing to predict counterfactual scenarios by deconfounding the state transitions.
By contrast, 
our proposed CWMs take into account the confounding factors between state transitions and can therefore internally predict the alternative futures not encountered in the past.

Understanding intuitive physics from visual perception has attracted considerable attention in machine learning and reinforcement learning society~\cite{Kubricht2017,10.5555/3045390.3045437,DBLP:conf/nips/WuYLFT15,NIPS2017_6620,10.1007/978-3-030-01252-6_20,Sun2019RelationalAF}.
Many of the existing models require supervised modeling of the object definition,
by either comparing the activation spectrum generated from neural network filters with existing types~\cite{garnelo2016towards}
or leveraging the bounding boxes generated by standard object detection algorithms in computer vision~\cite{keramati2018strategic}.
Although~\cite{zambaldi2018deep} have used the relational mechanism to discover and reason about relevant entities,
their model needs additional supervision to label entities with location information.
On the contrary, CWMs use a fully unsupervised manner to extract object abstractions.
To the best of our knowledge,
CoPhyNet~\cite{Baradel2020CoPhy} is the only work considering counterfactual scenario in learning physical dynamics but used direct supervision of the object positions,
thus only performing counterfactual forecasting in low-dimensional settings.
Nevertheless,
we still benefit from their proposed evaluation benchmark and demonstrate the effectiveness of our fully unsupervised causal world models in \Secref{sec:experiment}.

A variety of object-based World Models (WMs) have been proposed~\cite{greff2017neural,van2018relational,watters2019cobra} thanks to recent works studying the problem of object discovery from visual data~\cite{Burgess2019MONetUS,DBLP:conf/iclr/0003U0T17,Engelcke2020GENESIS,DBLP:conf/icml/GreffKKWBZMBL19,janner2018reasoning,NIPS2018_8079,10.1007/978-3-030-01252-6_20,Sun2019RelationalAF,NIPS2017_7040,xu2019unsupervised,DBLP:conf/uai/ZhengL0T18}.
By exploiting WMs' ability to \emph{think and plan ahead}~\cite{ha2018worldmodels,NIPS2018_8187},
model-based reinforcement learning algorithms have been shown to be more effective than model-free alternatives in certain tasks~\cite{Gu:2016:CDQ:3045390.3045688,pmlr-v80-igl18a,watter2015embed,levine2016end}.
However,
these models are often bottlenecked by the \emph{credit-assignment} problem:
they typically optimize a prediction or reconstruction objective function in pixel space and thereby could ignore visually small but informative features for predicting the future (such as, for instance, a bullet in an Atari game~\cite{Kaiser2020Model}).
On the contrary,
CWMs learn a set of object-centric abstract state variables and model the transition using graph neural networks~\cite{NIPS2016_6418,DBLP:journals/corr/abs-1806-01261,pmlr-v80-kipf18a,kipf2017semi,DBLP:journals/corr/LiTBZ15,10.1109/TNN.2008.2005605,wang2018nervenet,pmlr-v80-sanchez-gonzalez18a} by optimizing an energy-based hinge loss~\cite{LeCun06atutorial} in the latent space directly.

\begin{figure*}[t!]
\centering
\begin{subfigure}{0.3\textwidth}
\centering
\includegraphics[height=0.4\textwidth]{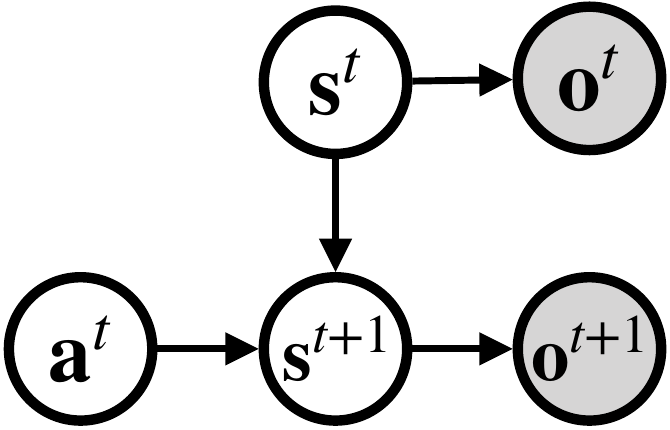}
\caption{}
\label{subfig:graph_noconf}
\end{subfigure}
\quad
\begin{subfigure}{0.3\textwidth}
\centering
\includegraphics[height=0.4\textwidth]{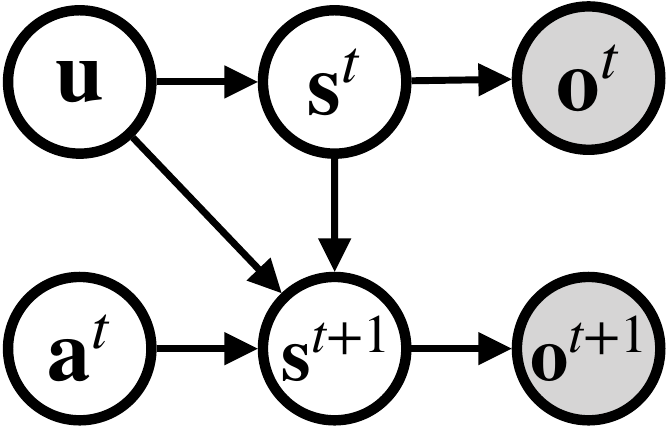}
\caption{}
\label{subfig:graph_conf}
\end{subfigure}
\quad
\begin{subfigure}{0.3\textwidth}
\centering
\includegraphics[height=0.4\textwidth]{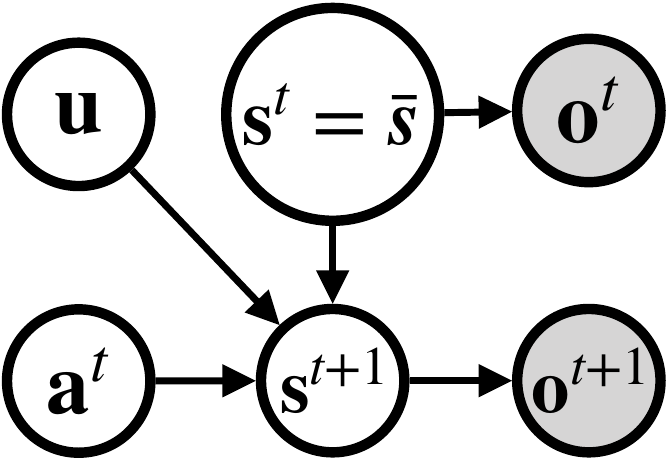}
\caption{}
\label{subfig:graph_do_s}
\end{subfigure}
\caption{
The graphical model of (a) POMDPs used by conventional WMs, (b) causal POMDPs used by CWMs, and (c) CWMs after intervention (\textbf{do}-operation).
}
\label{fig:graph}
\end{figure*}

Our proposed model adopts counterfactual learning to mitigate the problem of biased historical data.
Counterfactual Risk Minimization (CRM) \cite{swaminathan2015batch, swaminathan2015counterfactual}
is an instance of causal inference closely related to off-policy evaluation~\cite{kallus2018policy}.
Because the distribution of factual and counterfactual are not fully consistent,
causal models built upon historical data will produce bias.
Existing work has therefore considered unbiased learning by inverse propensity scores (IPS) methods~\cite{rosenbaum1983central}.
Most existing literature focuses on the scenarios of binary or limited discrete interventions~\cite{athey2017efficient},
including the popular \emph{doubly robust} frameworks ~\cite{dr1,dr3,dr2}.
Instead of measuring the causal effect of the  intervention in a discrete event space, 
our work focuses on the intervention distribution in continuous space.
In order to solve the problems of zero propensity scores on continuous intervention setting,
\citet{kallus2018policy} used kernel functions to achieve a smooth policy learning process, while \citet{swaminathan2015counterfactual} regularized the empirical risk via variance penalization.
Further discussion on this topic can be found in~\cite{louizos2017causal,DBLP:journals/corr/abs-1909-05299,swaminathan2015batch,swaminathan2015self}. 
Staying different from above approaches, the propensity weight in our method is not decided by the environments state,
but decided by the historical sampling strategies.

\section{Causal World Models}
Our goal is to build a world model capable of creating the dream environment by predicting the effect of interventions that \emph{would have happened}.
We start by introducing the notation and the problem definition of causality in learning physical dynamics.
Then, we introduce the general framework for learning object-oriented state abstractions and estimating the confounders.
Lastly, we introduce the usage of \emph{doubly robust functions} - a key technique in robust statistics and efficiency theory~\cite{dr3} - 
to enhance the sample efficiency and reduce the bias induced by the historical sample policy.

\subsection{Preliminary: POMDPs}
\label{subsec:pomdp}
As shown in~\Figref{subfig:graph_noconf},
conventional WMs typically consider the environment as a partially observable Markov Decision Process (POMDP) represented by the tuple 
$\Gamma={\langle} \mathcal{S}, \mathcal{A}, \mathcal{O}, \mathcal{U}, \mathcal{T}, \mathcal{R}, {T} {\rangle}$,
where $\mathcal{S}, \mathcal{A}, \mathcal{O}, {T}$
are the state space, the action space, the observation space, and the horizon, respectively.
For an agent taking actions in this environment, we consider the variable representing its received observation at time step $t$ be designated by $\rvo^t$,
the real-world observed values by 
$\vo^t \in \mathcal{O} \equiv \mathbb{R}^{D}$.
At each time step $t$, we denote as $\rvs^t$ the variable representing the abstract state of the world and $\vs^t \in \mathcal{S} \equiv \mathbb{R}^d$ the real-world value of this variable.
$\rvs^t$ can be typically regarded as the latent static component of the world to render the observation,
e.g.,
the shape, size and color of the objects.
The observation $\vo^{t}$ are provided by the environment following the observation distribution
$\mathcal{U}(\rvo^{t}|\rvs^{t}): \mathcal{S} \rightarrow \mathcal{O}$.
When the environment receives an action $a^{t}\in\mathcal{A}$ executed by the agent, 
it moves to a new state $\vs^{t+1}$ following the transition distribution
$\mathcal{T}(\rvs^{t+1}|\rvs^t,a^t): \mathcal{S} \times \mathcal{A} \rightarrow \mathcal{S}$ and returns a reward $r^t$ according to $\mathcal{R}(r^{t}|\rvs^{t},a^{t}): \mathcal{S} \times \mathcal{A} \rightarrow \mathbb{R}$.
Conventional WMs focus on estimating the distribution of $\rvs^{t+1}$ given the \emph{observed} value of $\rvs^{t}$ and $a^t$
\footnote{For the ease of understanding we only illustrate the learning of the transition function,
although the same analysis applies to the reward function.}.
This estimation will give us the observational conditional $p(\rvs^{t+1}|\rvs^{t},a^t)$.
We will see below that the observational conditional is generally biased in the real world where confounding factors widely exist.

\subsection{Causal POMDPs}
\label{subsec:causal_pomdp}
In this paper,
we extend the above typical setting of a POMDP by considering the existence of confounding factors $\rvu$~\cite{10.5555/2074394.2074401},
which are time-invariant hidden variables that influence both the intervention $\rvs^{t}$ and the outcome $\rvs^{t+1}$ as shown in~\Figref{subfig:graph_conf}.
In physical systems,
we regard $\rvu$ as a set of static but visually unobservable quantities such as object masses, friction coefficients, direction and magnitude of gravitational forces that cannot be uniquely estimated from a single time step.
Building world models upon causal POMDPs enables us to control the state $\rvs^{t}$ to create the \emph{dream} environment and facilitate imagination.
In other words,
we intend to know what the scenario $\rvs^{t+1}$ world have been if we \emph{set} the world state $\rvs^{t}$ to a specific value $\bar{\vs}$.
We adopt the \textbf{do}-operator  $\textbf{do}(\rvs^{t}=\bar{\vs})$ from causal reasoning~\cite{10.5555/1642718},
and arrives at the interventional conditional $p(\rvs^{t+1}|\textbf{do}(\rvs^{t}=\bar{\vs}),a^t)$.
In most real-world cases,
the observational conditional and the interventional conditional are different because of the existence of confounding factors $\rvu$.
To illustrate the impact, we provide simple calculations for 
the observational conditional
\begin{align*}
p(\rvs^{t+1}|\rvs^{t},a^t) 
&= \int_{\mathcal{\rvu}} p(\rvs^{t+1}|\rvu, \rvs^{t},a^t) p(\rvu | \rvs^{t},a^t) d \rvu \\
&= \int_{\mathcal{\rvu}} p(\rvs^{t+1}|\rvu, \rvs^{t},a^t) \frac{p(\rvs^{t},a^t | \rvu)}{p(\rvs^{t},a^t)} p(\rvu)d \rvu,
\end{align*}
and the interventional conditional
\begin{equation*}
p(\rvs^{t+1}|\textbf{do}(\rvs^{t}=\bar{\vs}),a^t) 
= \int_{\mathcal{\rvu}} p(\rvs^{t+1}|\rvu, \rvs^{t}=\bar{\vs},a^t) p(\rvu) d \rvu 
.
\end{equation*}
Clearly the above two are not the same  due to the influence of the confounders $\rvu$: $p(\rvs^{t},a^t | \rvu)$. This result is also termed as the Simpson's paradox~\cite{simpson1951interpretation}, a classical example for the existence of confounding in medical treatment scenario (see Appendix~\ref{appdx:pearl_example} for details).
The above observations therefore encourage us to build the world model following the causal POMDPs in~\Figref{subfig:graph_conf},
which enables the world model to disentangle the true effect of an intervention on the observation~\cite{louizos2017causal}.

\subsection{Learning Causal World Models}
\label{subsec:cwms_learning}

To learn the transition function of the dream world,
we apply the \emph{do-intervention}~\cite{10.5555/1642718} $\textbf{do}(\rvs^t=\bar{\vs}^t)$ on the abstract state variable as shown in~\Figref{subfig:graph_do_s},
where $\bar{\vs}^t \in \mathcal{S} \equiv \mathbb{R}^d$ is the counterfactual value in the dream environment.
The intervened abstract state is then rendered as an object-level observable change applied to $\vo^t$ (such as, for instance, object displacement or removal) by the conditional observation distribution $\mathcal{U}(\rvo^{t} =\bar{\vo}^t| \textbf{do}(\rvs^t=\bar{\vs}^t))$,
where $\bar{\vo}^t \in \mathcal{O} \equiv \mathbb{R}^{D}$ represents the value of the counterfactual observation.

Learning the causal world models aims at answering:
\begin{quote}
\vskip -0.05in
Given that we have observed $\rvo^{t:T}=\vo^{t:T}$ in the real world,
what is the probability that $\rvo^{t + 1:T}$ would have been $\bar{\vo}^{t + 1:T'}$ if $\rvo^t$ were $\bar{\vo}^t$ in the dream world?
\end{quote}
\vskip -0.05in
Particularly,
having observed the tuple $(\vo^{t:T}, \bar{\vo}^{t})$,
we wish to predict positions and poses of all objects in the scene at time $t={T}'$.
We consider an \emph{off-policy} setting for the model training, where we operate solely on a buffer of offline experience obtained from,
such as, an exploration policy.
In this paper,
we focus on building an internal \emph{dream} environment to simulate the future not encountered in the past with any given dream policy $p(a^t|\rvs^t)$.
We therefore drop the action variable from the notation in the remainder of this paper by 
$
     \int_{a^t}\mathcal{T}(\rvs^{t+1}|\rvs^t,a^t) p(a^t|\rvs^t) d a^t = \int_{a^t}\mathcal{T}(\rvs^{t+1},a^t|\rvs^t) d a^t =  \mathcal{T}(\rvs^{t+1}|\rvs^t).
$

Instead of recovering the individual treatment effect (ITE),
which measures the causal effect when the intervention happens in a discrete event space~\cite{DBLP:conf/icml/AlaaS19,louizos2017causal,DBLP:journals/corr/abs-1909-05299}, 
we focus on the identification of the intervention distributions
$p(\rvo^{t + 1:{T'}} | \textbf{do}(\rvs^t = \bar{\vs}^t))$ and $p(\rvo^{t} | \textbf{do}(\rvs^t = \bar{\vs}^t))$ in the continuous space.

\begin{definition}
We say a variable $\rvu$ factorizes the distribution of variable $\rvs^t$ iff:
$
p(\rvs^t)=\int_{\mathcal{\rvu}} p\left(\rvs^t | \rvu=\vu\right) p(\rvu=\vu) d \rvu.
$
\end{definition} 
\begin{assumption}
There exists an unobserved variable $\rvu$ such that 
(i) $\rvu$ blocks all backdoor path\footnote{In a directed acyclic graph (DAG), a path that connects $X$ to $Y$ is a backdoor path\ from $X$ to $Y$ if it has an arrowhead pointing to $X$,
e.g. $X\leftarrow Z \rightarrow Y.$}
from $\rvs^t$ to $\rvs^{t + 1}$ and (ii) there exist no backdoor path from the abstract state $\rvs^t$ to the observation $\rvo^t$.
\end{assumption}
\begin{theorem}Under Assumption 1, if there exists an estimator of $\rvu$,
$\hat{\rvu}(\rvs^t)$, such that
$
\hat{\rvu}(\rvs^t) \stackrel{a \cdot s}{\longrightarrow} \rvu,
$
then the intervention distributions $p(\rvo^{t + 1:{T'}} | \textbf{do}(\rvs^t = \bar{\vs}^t))$ and $p(\rvo^{t} | \textbf{do}(\rvs^t = \bar{\vs}^t))$ are identifiable.
\end{theorem}
\begin{proof}
By applying \textbf{do}-calculus to the causal graph in~\Figref{subfig:graph_conf} (as shown in~\Figref{subfig:graph_do_s}),
referred to back-door criterion, the intervention distributions are then identified by
\begin{align}
    &p(\rvo^{t} | \textbf{do}(\rvs^t = \bar{\vs}^t)) = p(\rvo^{t} | \rvs^t = \bar{\vs}^t),\nonumber \\
    &p(\rvo^{t + 1:{T'}} | \textbf{do}(\rvs^t = \bar{\vs}^t))
    = \int_{\mathcal{\rvu}} p(\rvo^{t + 1:{T'}} | \rvs^{t + 1:{T'}}) \nonumber \\
    &\quad\times p(\rvs^{t + 1:{T'}} | \rvs^t = \bar{\vs}^t, \hat{\rvu}(\rvs^t)=\vu) p(\hat{\rvu}(\rvs^t)=\vu) d \rvu
\label{eq:intervention_identifiability}
\end{align}
This completes the proof since the quantities in the final expression of~\Eqref{eq:intervention_identifiability} can be identified from the observed distribution $p(\rvo^{t:{T}})$ and $p(\bar{\rvo}^t)$ and the estimator $\hat{\rvu}(\rvs^t)$.
\end{proof}
We regard $\rvu$ as the \emph{confounder} described in~\Secref{sec:intro}.
In this paper,
we consider $\rvu$ as a set of static but visually unobservable quantities such as object masses, friction coefficients, direction and magnitude of gravitational forces.
We develop CWMs by learning $\hat{\rvu}$ to estimate the latent representation of $\rvu$.
The dream rollout $\bar{\vo}^{t+1:{T'}}$ is then predicted given $\bar{\vo}^t$ and $\hat{\rvu}$ following~\Eqref{eq:intervention_identifiability}.

\subsection{Model Design} 
\label{subsec:cwm_design}
\begin{figure*}[t!]
\centering
\includegraphics[width=\textwidth]{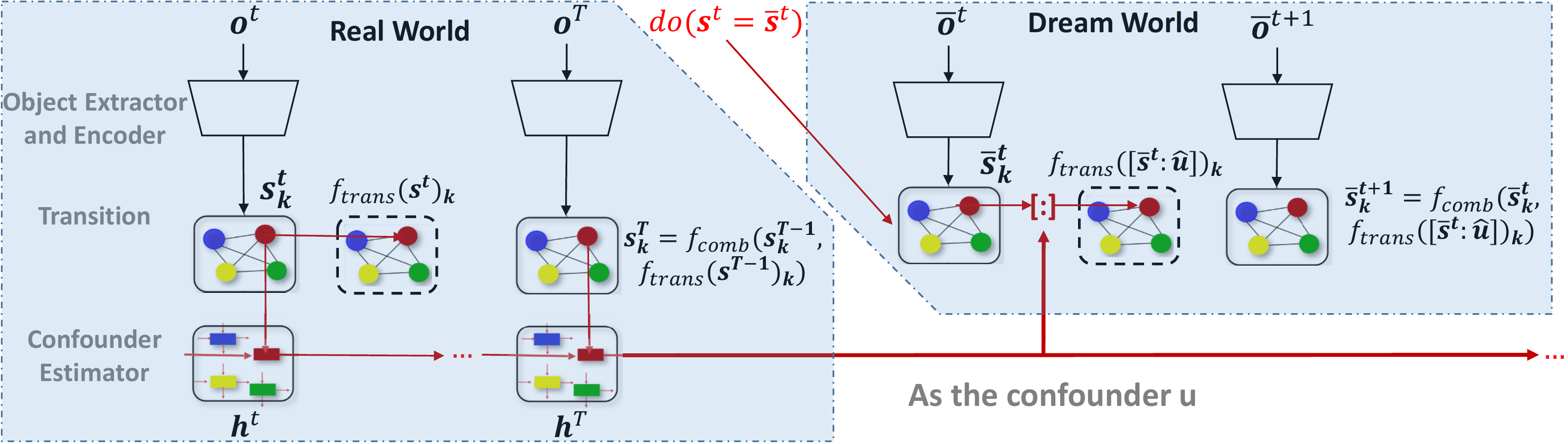}
\caption{An illustration of the proposed Causal World Models.
}
\label{fig:cwm}
\end{figure*}
Our starting point is learning a set of object-oriented state abstractions and models the transition using graph neural networks,
which facilitates capturing the structural property of the physical system and tackles the credit-assignment problem faced by most conventional generative WMs.
To further improve the discrepancy between state abstractions,
the training of CWM is carried out with an energy-based hinge loss~\cite{LeCun06atutorial},
which scores positive against negative experiences in the form of transition dynamics and propose to learn the transition dynamics in the latent space directly.
Lastly, we introduce our confounder estimator using a set of recurrent neural networks.

\paragraph{Object Extractor and Encoder}
\label{subsubsec:obj_enc}
To avoid the credit-assignment problem introduced by generative-based world models,
which optimize a prediction or reconstruction objective in pixel space and thereby could ignore visually small but informative features for predicting the future (such as, for instance,  a bullet in an Atari game~\cite{Kaiser2020Model}),
CWM learns a set of abstract state values $\rvs^t \in \mathcal{S} \equiv \mathbb{R}^D$ for each object in $\vo^t \in \mathcal{O}$.
Formally, we have an \emph{encoder}
$f_{\text{enc}}: \mathcal{O}\rightarrow K \times \mathcal{S}$
which maps observation to $K$ abstract state representations.
The choice of $K$ is a hyperparameter and we conduct ablation study on its influence in Appendix~\ref{appdx:hyperparam}.
A simple implementation of $f_{\text{enc}}$ consists of two modules:
1) a CNN taking $\vo^t$ as input with $K$ feature maps in its last layer corresponding to $K$ object slots, respectively;
and 2) an MLP taking each feature map as input and producing the corresponding abstract state representation $\vs^t_k$ with $k \in \{1,2,\dots,K\}$.
Note that choosing other advanced unsupervised object representation learning methods for $f_{\text{enc}}$ is a straightforward extension.
For example,
an alternative choice could be the Transporter~\cite{NIPS2019_9256},
which is used for learning concise geometric object representations in terms of image-space coordinates in a fully unsupervised manner.

\paragraph{Transition Estimator} 
We formulate the \emph{transition model} of CWM as $f_{\text{trans}}: K \times \mathcal{S}\rightarrow K \times \mathcal{S}$ operating on abstract state representations $\vs^t$.
In this paper,
we implement $f_{\text{trans}}$ as a graph neural network~\cite{DBLP:journals/corr/abs-1806-01261,kipf2017semi,DBLP:journals/corr/LiTBZ15,10.1109/TNN.2008.2005605},
which allows us to model pairwise interactions between object states while being invariant to the order in which objects are represented.
The transition model $f_{\text{trans}}$ takes as input the tuple of abstract object representations $\vs^t = (\vs^t_1, \dots, \vs^t_K)$ and predicts updates $\Delta \vs^t=(\Delta \vs^t_1,\dots,\Delta \vs^t_K)$.
which are used to obtain the next abstract state representations via a combination function $f_{\text{comb}}(\cdot,\cdot)$.
In this work we simply choose $f_{\text{comb}}$ as an addition function,
i.e.,
$f_{\text{comb}}\left(\vs^t, f_{\text{trans}}(\vs^t)\right) = \vs^t + \Delta \vs^t$.
Alternatively, one could model the combination using other graph embedding methods~\cite{DBLP:conf/icml/NickelTK11,trouillon2016complex}.
The graph neural network consists of node update functions $f_{\text{trans-node}}$ and edge update functions $f_{\text{trans-edge}}$ with shared parameters across all nodes and edges.
These functions are implemented as MLPs and we choose the message passing updates as
$\Delta \vs^t_j = f_{\text{trans-node}}([\vs^t_j,\textstyle\sum_{i\neq j} \ve_{i, j}^t])$,
where
$\ve_{i, j}^t = f_{\text{trans-edge}}([\vs^t_i,\vs^t_j])$
is an intermediate representation of the interaction between nodes $i$ and $j$.
We denote the output of the transition model for the $k$-th object as $\Delta \vs^t_k = f_{\text{trans}}(\vs^t)_k$ in the following.

\paragraph{Training Objective}
To further improve the discrepancy between state abstractions,
the training of CWM is carried out with an energy-based hinge loss~\cite{LeCun06atutorial},
which scores positive and negative samples in a different direction and has been widely used in the field of graph representation learning~\cite{bordes2013translating,grover2016node2vec,perozzi2014deepwalk,schlichtkrull2018modeling,velickovic2018graph}.
We define the energy of two consecutive state variables and the energy of the negative sample as 
\begin{align}
\begin{split}
\mathcal{H} &= \frac{1}{K}\sum\nolimits_{k=1}^K d(f_{\text{comb}}\left(\vs^t_k, f_{\text{trans}}(\vs^t)_k\right), \vs^{t+1}_k) \text{ ,}
\\
\tilde{\mathcal{H}} &= \frac{1}{K}\sum\nolimits_{k=1}^K d(\tilde{\vs}^t_k, \vs^{t+1}_k),
\end{split}
\label{eq:energy}
\end{align}
where $d(\cdot,\cdot)$ denotes the squared Euclidean distance
and $\tilde{\vs}^t$ denotes a corrupted abstract state encoded by $\tilde{\vs}^t = f_{\text{enc}}(\tilde{\vo}^t)$ using a random sample $\tilde{\vo}_t$ from the experience buffer.
The objective of CWM then takes the following energy-based hinge loss form as:
\begin{equation}
\mathcal{L} = \mathcal{H} + \max(0, \gamma - \tilde{\mathcal{H}})\, ,
\label{eq:obj_function}
\end{equation}
where margin $\gamma$ is a hyperparameter.
The overall loss is to be understood as an expectation of the above over samples from the experience buffer.

\paragraph{Deconfounding}
We develop CWM by learning an estimator $\hat{\rvu}$ to approximate the latent representation of the confounders $\rvu_k$ for each object $k$.
The estimator $\hat{\rvu}$ is trained end-to-end by optimizing the counterfactual prediction loss through a recurrent neural network,
which takes as input the sequence of abstract state variables $\vs^{0:{T}}_k$.
Concretely, we run a dedicated RNN with GRU~\cite{conf/emnlp/ChoMGBBSB14} $f_\phi$ for each object trajectory $\vs^{0:{T}}_k$ and keep the last hidden state $h^{T}_k=f_\phi(\vs^{0:{T}}_k)$ as the estimate of the latent representation of the confounders,
i.e., $\hat{\rvu}_k \triangleq h^{T}_k$.
We follow the convention of sharing the parameters of $f_\phi$ over $K$ objects~\cite{Baradel2020CoPhy},
which makes the model invariant to the number of objects.
The confounders estimation is then fed into the dream world model by concatinating (shown as \textbf{[:]} in~\Figref{fig:cwm}) with $\vs^t$ before the transition estimator to predict the alternative futures after \textbf{do}-intervention.

\subsection{Doubly Robust Learning from Historical Observations}
\label{subsec:crm}

In real-world scenarios,
models trying to approximate the counterfactual distribution will be biased because of the inadequate counterfactual records.
For example,
in the physical dynamics scenario with continuous intervention space,
historical trajectories satisfy a certain but unknown distribution which corresponds to the historical sampling policy.
It cannot be discretely considered like propensity scores in traditional IPS function~\cite{IPSfunction}, since it may involve an unavoidable variance.
To minimize counterfactual risk when training CWMs,
we choose the propensity score as the historical sampling policy of selected intervention $p(\mathbf{s^0})$ represented by a density function.
Inspired by the doubly robust (DR) estimator,
our estimator is set as below.
\begin{definition}
Let $p(\vs^0)$ denotes the history observation distribution (propensity score) on $t_0$,  and $O(\mathbf{s}|t) = \delta(\mathbf{s} = \mathbf{s^t})$ denotes the observation indicator function (i.e. when $\mathbf{s} = \mathbf{s^t}$, the $O(\mathbf{s^{t}}|t) = 1$), the Doubly Robust prediction of $\vs^{t+1}$ is
\begin{align}
\begin{split}
    \hat{\vs}_{DR}^{t+1}(\vs^t) &=  \frac{O(s = \vs^{0}|t_0)}{p(\vs^0)}(\vs^{t+1}-\vs^{t}- \hat{f}_{\text{trans}}(\vs^{t}))\\
    &\quad\quad+(\vs^{t} + \hat{f}_{\text{trans}}(\vs^{t}))
\end{split}
\end{align}
\end{definition}
\begin{proposition}
Given the propensity score $p(\mathbf{s^0})$, the doubly robust oracle is unbiased against the true trajectory observation,
i.e.,
$\mathbb{E}[\hat{\mathbf{s}}_{DR}^{t+1}] = \mathbf{s^{t+1}}$.
\end{proposition} 
See Appendix~\ref{appdx:prop1_proof} for the proof.
We arrive at the doubly robust objective function for CWMs as:
\begin{equation}
    \mathcal{L}_{DR} = \frac{1}{K}\sum\nolimits_{k=1}^K d(\hat{\vs}_{DR,k}^{t+1}, \vs^{t+1}_k)  + max(0, \gamma - \tilde{\mathcal{H}}),
    \label{eq:unbias_function}
\end{equation}
where $\gamma$ is a free parameter and $\tilde{\mathcal{H}}$ is the energy of negative samples as described in~\Eqref{eq:obj_function}.

\section{Experiment}
\label{sec:experiment}
We compare our Causal World Models (CWMs) against the state-of-the-art conventional world models (WMs)~\cite{Kipf2020Contrastive} in two benchmarks,
CoPhy~\cite{Baradel2020CoPhy} and PHYRE~\cite{bakhtin2019phyre},
to evaluate the quality and the usability of the created \emph{dream} environment,
respectively.
We empirically evaluate our method and show reductions in sample complexities for reinforcement learning tasks and improvements in counterfactual physical reasoning predictions.

\begin{figure*}[t!]
\centering
\includegraphics[width=\textwidth]{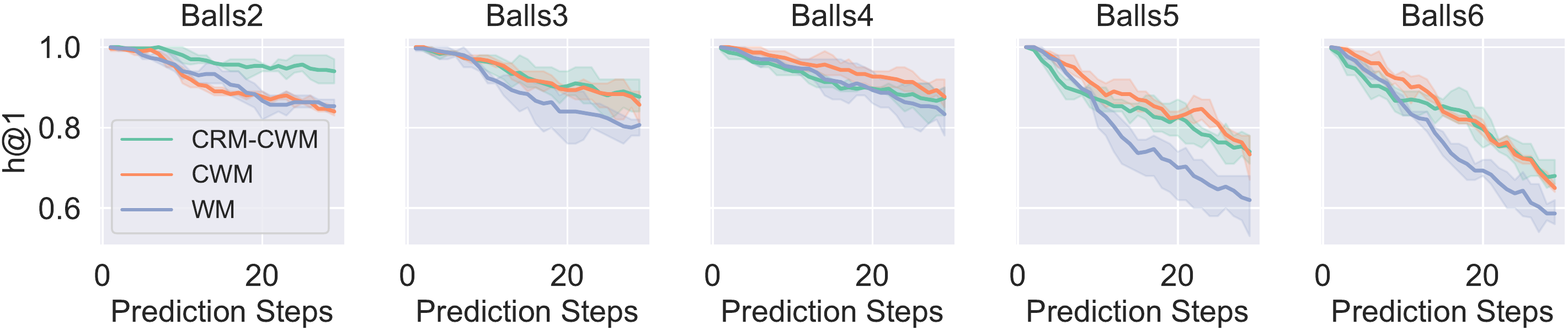}
\label{subfig:h1}
\caption{
(H@1) Ranking results for multi-step prediction in latent space in different environments. Our models (CWM and CRM-CWM) consistently achieve the best result.
}
\label{fig:exp_metrics_h1}
\end{figure*}
\begin{figure*}[t!]
\centering
\includegraphics[width=\textwidth]{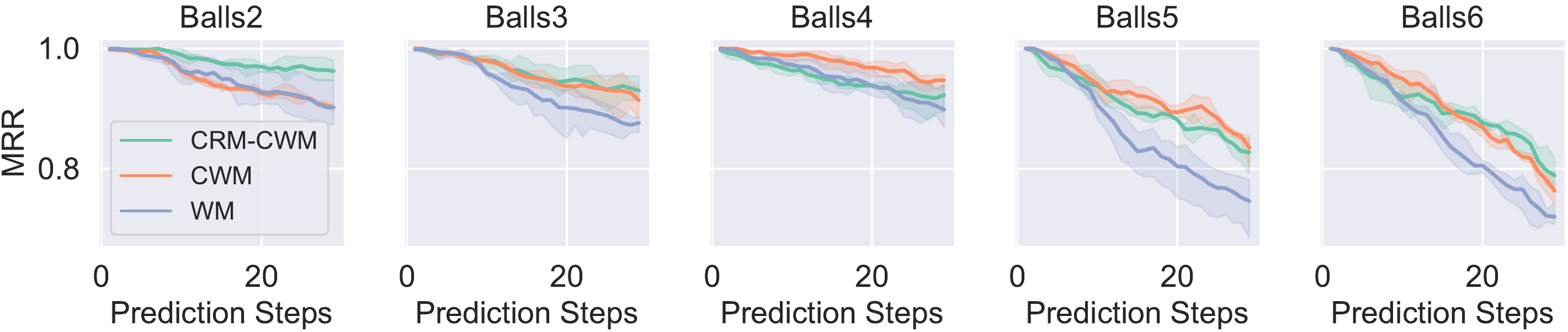}
\label{subfig:mrr}
\caption{
(MRR) Ranking results for multi-step prediction in latent space in different environments. Our models (CWM and CRM-CWM) consistently achieve the best result.
}
\label{fig:exp_metrics_mrr}
\end{figure*}

\subsection{Dream Quality}
\label{subsec:exp_cophy}
\paragraph{Environment Settings}
We conduct our experiments on CoPhy~\cite{Baradel2020CoPhy},
a recently proposed benchmark suite for counterfactual reasoning of physical dynamics from raw visual input,
which contains $K$ balls $(K = 2,3,4,5,6)$ initialized with a random position and velocity.
The world is parameterized by a set of visually unobservable quantities, or confounders,
consisting of ball masses and the friction coefficients.
As described in~\Secref{subsec:causal_pomdp},
the training data of CoPhy is a tuple of two trajectories:
the real-world trajectory $\vo^{0:{T}}$ and the counterfactual trajectory $\bar{\vo}^{0:{T}'}$ after one of the two \textbf{do}-operators: ball displacement or removal.
The model is evaluated by its dream creation quality,
i.e.,
the prediction of the counterfactual outcome $\bar{\vo}^{1:{T}'}$ given real-world trajectory $\vo^{0:{T}}$ and the intervened initial observation $\bar{\vo}^{0}$.
In our experiments, we sampled $700$ tuples of data from the CoPhy benchmark as the training set and $300$ tuples of data as the test set.
All trajectories contains observations for $29$ time steps.
We follow the convention of using two ranking metrics: (1) Hits @ Rank 1 (H@1) and (2) Mean Reciprocal Rank (MRR) to evaluate model performance directly in the latent space~\cite{bordes2013translating,Kipf2020Contrastive}.
The predicted abstract state representation is compared to the encoded ground truth observation and a set of reference states,
which are encoded from random observations sampled from the experience buffer.
We report the average scores over the test set for different prediction steps $T'$.
Details on evaluation metrics can be found in Appendix~\ref{appdx:eval_metrics}.

\paragraph{Model Settings}
We compare with the state-of-the-art world model (WMs) variations for representation learning in environments with compositional structure~\cite{Kipf2020Contrastive}.
The counterfactual prediction (dream creation) of WMs is implemented by simply concatenating $\vo^{t} $ with $\bar{\vo}^{t}$ following the causal graph in~\Figref{subfig:graph_noconf}.
We propose two variations of Causal World Models (CWMs):
the original CWMs described in~\Secref{subsec:cwm_design} and the unbiased-augmented CRM-CWMs using counterfactual risk minimization as described in~\Secref{subsec:crm}.
CRM-CWMs work by estimating a Gaussian probability density function of historical samples on the training set. This density function is then used as the propensity score for~\Eqref{eq:unbias_function}.
All models share the same setting for the object extractor, encoder, and the transition estimator.
Details on architecture and hyperparameters setting can be found in Appendix~\ref{appdx:hyperparam_cophy}.
We choose over different number of slots $\{2,3,4,5,6\}$ (the hyperparameter $K$ as described in~\Secref{subsubsec:obj_enc}) and present the result with the best configuration for each model respectively.
Ablation study on the number of slots can be found in Appendix~\ref{appdx:exp_results}.

\paragraph{Result Analysis}
We firstly show the qualitative example of predicted dream trajectories of WMs and CWMs in the latent space of the 6-ball CoPhy environment in ~\Figref{subfig:traj_wm} and~\Figref{subfig:traj_cwm}, respectively.
Trajectories are projected to two dimensions via PCA and presented with the same scale.
The mean squared errors between the predicted state values and ground truth are presented for each episode.
We can observe that CWMs reliably predict the trajectories in the dream world without direct supervision,
while WMs suffer from a poor counterfactual prediction ability.
Further qualitative results can be found in Appendix~\ref{appdx:exp_results}.
\Figref{fig:exp_metrics_h1} and \Figref{fig:exp_metrics_mrr} show the ranking results of predicting unseen future in the dream world in terms of H@1 and MRR metrics respectively.
Our proposed CWMs and CRM-CWMs consistently give the best result across all training environments,
demonstrating that the estimation of confounders helps counterfactual prediction by modeling the intrinsic environment property.
Also, we discover that the conventional WMs are sensitive to the difficulty of environments:
the hidden confounders of the environment results in higher variance on such methods learning directly from observational data.
On the other hand,
our proposed CRM-CSMs consider the historical sample policy and use counterfactual learning to reduce the sampling bias during the learning process,
thus achieving the best result among the comparison group.

\subsection{Dream Usability}
\label{subsec:exp_phyre}
\begin{figure*}[t!]
\centering
\begin{subfigure}{0.65\textwidth}
\centering
\includegraphics[width=\textwidth]{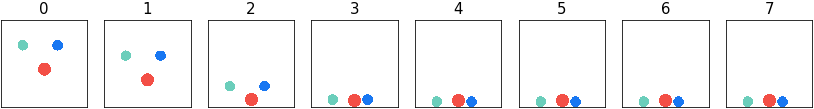}
\caption{A failed trial: placing the red cannot make the green touch the blue.}
\label{subfig:success_phyre}
\includegraphics[width=\textwidth]{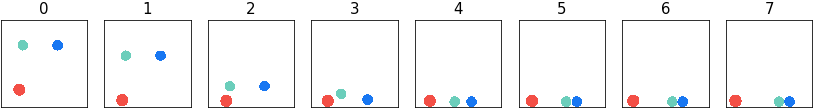}
\caption{A successful trial where the green and blue ball touch each other.}
\label{subfig:fail_phyre}
\end{subfigure}
\begin{subfigure}{0.3\textwidth}
\centering
\includegraphics[width=\textwidth]{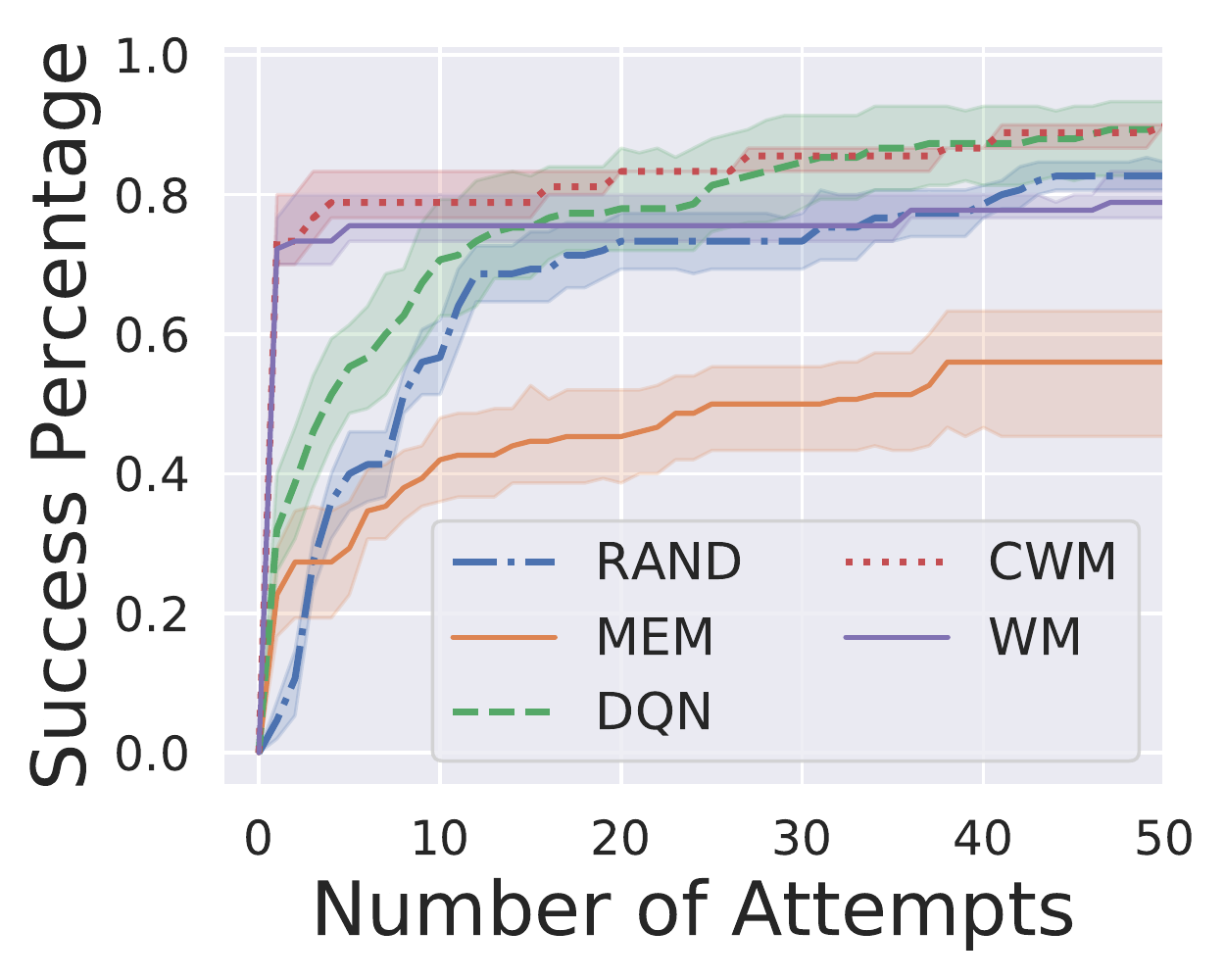}
\caption{}
\label{subfig:rl}
\end{subfigure}
\caption{
Example of the PHYRE environment with (a) a failed trial and (b) a successful trial.
(c) Percentage of solved tasks with respect to the number of attempts per task on PHYRE environment.
}
\label{fig:phyre}
\end{figure*}

\paragraph{Environment Settings}
We evaluate the usability of the created \emph{dream} world by measuring the number of interactions with the environment needed to achieve the goal in PHYRE~\cite{bakhtin2019phyre} environment.
The task is formulated as a physics puzzle in a simulated 2D world containing multiple objects as shown in~\Figref{subfig:fail_phyre}.
Each task has a goal state defined as a (subject, relation, object) triplet identifying a relationship between two objects that the agent needs to achieve before it reaches the horizon $T$,
e.g., make the green ball touch the blue ball as shown in~\Figref{subfig:success_phyre}.
The agent's action,
as well as the intervention,
is defined as placing an object into the world before the environment rollout,
i.e.,
$\mathcal{A} = (a_x, a_y, a_r) \equiv \mathbb{R}^D \times \mathbb{R}^D \times \mathbb{R}^1$,
where $(a_x, a_y)$ and $a_r$ represents the object's location and size respectively.
The agent takes no actions during the environment rollout and receives a binary reward indicating whether the goal is satisfied after the simulation.
All agents are trained to predict whether a specific action (from the default $100k$ action candidates) can solve a specific task
and evaluated by the number of attempts used to solve previously unseen tasks with a given set of action candidates.

\paragraph{Model Settings}
We compare with three baseline agents:
(1) random agent (RAND),
which samples actions uniformly at random from action space at test time;
(2) non-parametric “memorized” agent (MEM),
which computes the fraction of training tasks that a set of actions can solve,
sorts the action set according to the successful rate,
and tries each action in this order at test time;
(3) the Deep Q-network (DQN) agent,
which trains a deep network on the offline collected data to predict the reward for an observation-action pair.
We use the same setting as described in~\Secref{subsec:exp_cophy} for WMs and CWMs and train a separate classifier to predict whether the task is solved given an abstract state value.
At test time, WMs and CWMs produce a score through the classifier based on the predicted counterfactual state value for all valid action candidates of each task,
and attempt to solve the task in the score order.
Details on architecture and hyperparameters setting can be found in Appendix~\ref{appdx:hyperparam_phyre}.

\paragraph{Result Analysis}
\Figref{subfig:rl} presents success-percentage curves on the PHYRE environment averaged over all test tasks and 3 random seeds.
It is clear that CWMs agents perform better than WMs in terms of both the sample complexity and the final average success percentage within the trial budget (50 attempts),
which demonstrate the importance of deconfounding in physical simulation.
Meanwhile,
model-based agents (CWMs and WMs) can significantly reduce the sample complexity when compared to model-free counterparts,
owing to the fact that learning a model of the world can help to reason about what would happen upon a particular change to a previous attempt.
These, in turn, demonstrate the advantage of our model-based approach towards physical reasoning.

\section{Conclusion}
In this paper,
we propose Causal World Models (CWMs) to create a \emph{dream} world able to predict the alternative future not encountered in the real world in a fully unsupervised manner.
By learning an estimator of the latent confounders and optimizing directly within the abstract state space,
CWMs outperform state-of-the-art world models in both the reinforcement learning tasks and physical simulation environments.
To further reduce the inevitable bias of counterfactual dataset,
we also propose a counterfactual risk minimization method for CWMs and demonstrate its effectiveness in learning counterfactual physical dynamics.
In the future,
we would like to explore CWMs' ability in a broader range of applications and more complex environments.

\bibliographystyle{icml2020}
\bibliography{neurips}

\newpage
\onecolumn
\appendix
\setcounter{figure}{0}
\setcounter{definition}{0}
\setcounter{assumption}{0}
\setcounter{theorem}{0}
\setcounter{equation}{0}
\setcounter{page}{1}
\setcounter{footnote}{0}

\section{Simpson's Paradox in the Medical Treatment Scenario}
\label{appdx:pearl_example}
We include the classical example of Simpson's Paradox in the medical treatment scenario from~\citet{10.5555/1642718} to demonstrate the wide existence of confounding.

\begin{figure*}[t!]
\centering
\begin{subfigure}{0.27\textwidth}
\centering
\includegraphics[width=0.4\textwidth]{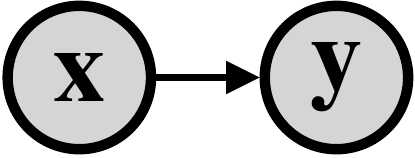}
\caption{}
\label{appdxsubfig:treatment_noconf}
\vskip 0.2in
\includegraphics[width=0.4\textwidth]{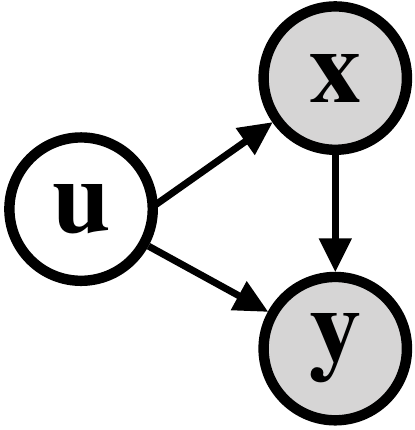}
\caption{}
\label{appdxsubfig:treatment_conf}
\end{subfigure}
\begin{subfigure}{0.7\textwidth}
\centering
\includegraphics[width=0.65\textwidth]{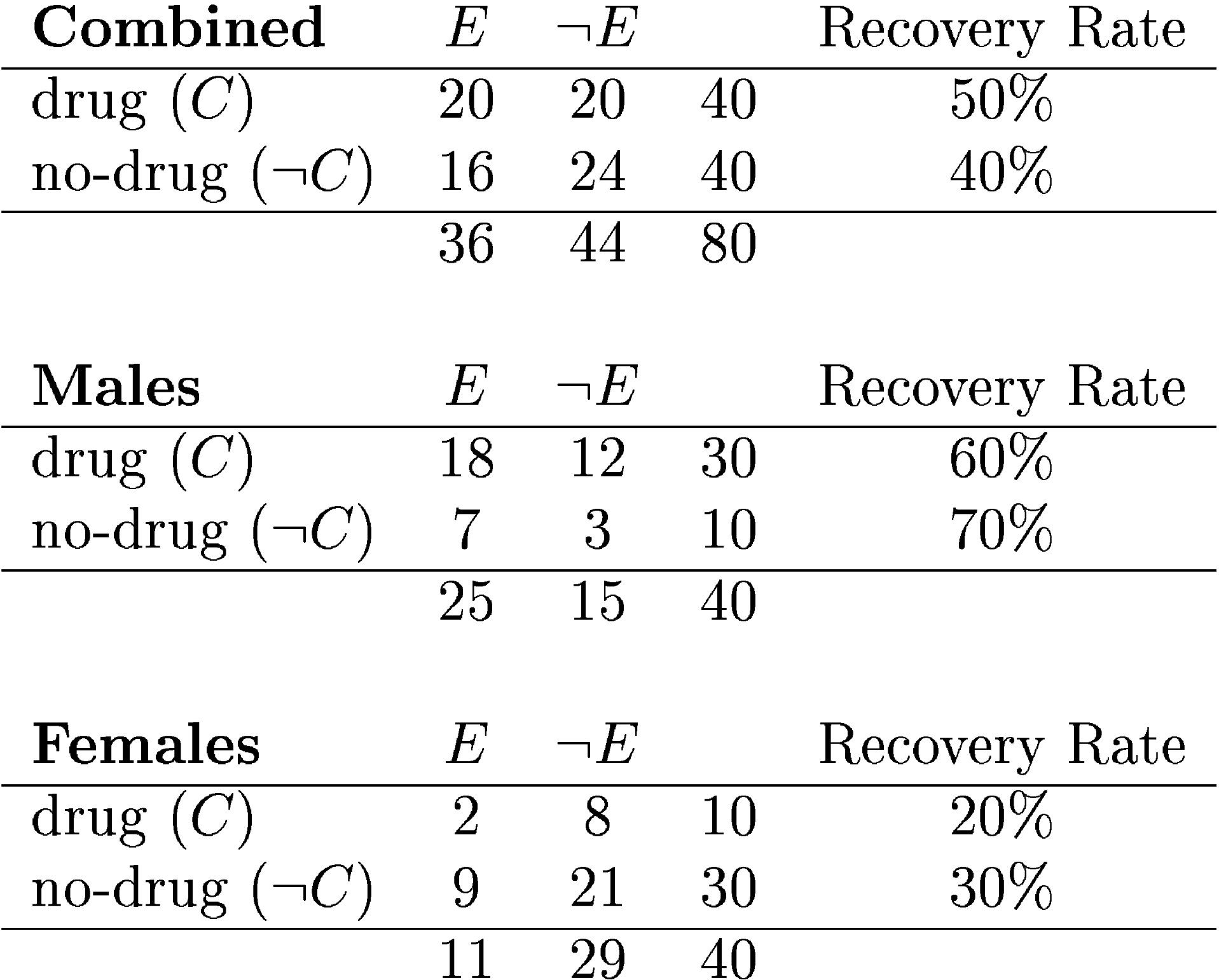}
\caption{Example from~\citet{10.5555/1642718}.}
\label{appdxsubfig:simpsons}
\end{subfigure}
\caption{
Confounding in medical scenarios.
We present the simplified causal graph of the transition model used by
(a) conventional World Models (WMs);
(b) Causal World Models (CWMs), with the consideration of confounding.
$\rvx$ represents the treatment (taking drug or not),
$\rvy$ represents the recovery,
and $\rvu$ represents the gender (male or female).
}
\label{appdxfig:treatment}
\end{figure*}

Consider the example where $\rvy$ is the observational recovery rate and $\rvx$ is taking a drug as the treatment.
As illustrated in~\Figref{appdxsubfig:simpsons},
\citet{10.5555/1642718} shows that the overall recovery rate of all patients increased from $40\%$ to $50\%$ after the treatment,
while the drug appeared to be harmful to both male (recovery rate dropped from $70\%$ to $60\%$) and female patients (recovery rate dropped from $30\%$ to $20\%$) respectively.
Simply using the causal chain of conventional world models (WMs) (shown as in~\Figref{appdxsubfig:treatment_noconf}) will lead to the above Simpson's paradox~\cite{simpson1951interpretation},
which refers to the phenomenon that a trend appears in several different groups of data but disappears or reverses when these groups are combined.
An intuitive explanation for this phenomenon is that females are more vulnerable to the disease and are much more likely to die without the treatment,
resulting in that no treatment looks worse than taking the drug overall.
This intuition leads us to the assumption of the true underlying causal graph of the example shown in~\Figref{appdxsubfig:treatment_conf},
which takes confounding $\rvu$ into account as the gender affects both the treatment and recovery.

\section{Proof of Proposition 1}
\label{appdx:prop1_proof}
\textbf{Proof:} Let $\vs^t$ be the $t$-th latent state of trajectory $k$.

\begin{align*}
    \mathbb{E}(\hat{\vs}_{DR}^{t+1}(\vs^t)|\vs^{0}) &= \mathbb{E}(\frac{O(\vs^{0}|t=0)}{p(\vs^0)}(\vs^{t+1}-\vs^{t}- \hat{f}_{\text{trans}}(\vs^{t}))+(\vs^{t} + \hat{f}_{\text{trans}}(\vs^{t}))|\vs^{0})\\
    &=\mathbb{E}(\frac{O(\vs^{0}|t=0)}{p(\vs^0)})\mathbb{E}(\vs^{t+1}-\vs^{t}- \hat{f}_{\text{trans}}(\vs^{t})|\vs^{0})+\mathbb{E}(\vs^{t} + \hat{f}_{\text{trans}}(\vs^{t})|\vs^{0})\\
    &=\vs^{t+1}
\end{align*}

\section{Evaluation Metrics for the Dream Quality Experiment}
\label{appdx:eval_metrics}
We follow the convention of using ranking metrics to evaluate model performance directly in latent space~\cite{bordes2013translating,Kipf2020Contrastive}.
The predicted abstract state representation is compared to the encoded ground truth observation and a set of reference states,
which are encoded from random observations sampled from the experience buffer.
We measure and report the following two metrics:
\paragraph{Hits @ Rank 1 (H@1)}
H@1 is a binary score measuring whether the rank of predicted abstract state representation equals to 1 after ranking all the reference state representations by distance to the ground truth.
\paragraph{Mean Reciprocal Rank (MRR)}
MRR measures the inverse average rank of all the by $MRR = \frac{1}{N}\sum_n \frac{1}{\text{rank}_n}$, where $\text{rank}_n$ is the rank of $n$-th sample.

We report the average of these scores over the test set for different prediction steps $T'$.

\section{Architecture and Hyperparameter Settings}
\label{appdx:hyperparam}
\subsection{Dream Quality Experiments}
\label{appdx:hyperparam_cophy}
We train all models on an experience buffer obtained by running a random policy on the respective environment.
We choose $700$ episodes with $29$ environment steps each for the training set and $300$ episodes with $29$ steps each for the test set.
All observations for this benchmark have been rendered into the visual space (RGB) at a resolution of $448 \times 448$ pixels with PyBullet~\cite{coumans2019} and we resize the observation to $50 \times 50$ pixels each.

All models are trained for $200$ epochs using the Adam~\cite{kingma2014adam} optimizer with a learning rate of $5e-4$ and a batch size of $25$ (due to the relatively large size of each training tuple).
All experiments were completed on a single NVIDIA GeForce GTX 1080 Ti GPU in under 2 hours.

\subsubsection{WMs}
\label{apdx_subsubsec:wms}
\paragraph{Object Extractor}
We choose the object extractor as the following.
\begin{enumerate}
\item $9 \times 9$ conv. 32 LeakyReLU~\cite{xu2015empirical}. padding 4. stride 1. BatchNorm~\cite{ioffe2015batch}
\item $5 \times 5$ conv. $K$ Sigmoid. padding 0. stride 5. BatchNorm
\end{enumerate}

\paragraph{Object Encoder}
After reshaping/flattening the output of the object extractor, we obtain a $100$-dim vector representation per object.
The object encoder is an MLP using the following architecture.
\begin{enumerate}
\item fully connected. 512 ReLU. 
\item fully connected. 512 ReLU. LayerNorm~\cite{ba2016layer}
\item fully connected. 4.
\end{enumerate}

\paragraph{Transition Estimator}
Both the node and the edge model in the GNN-based transition model are MLPs with the same architecture as the above object encoder module.

\paragraph{Loss Function}
We choose the margin in the hinge loss as $\gamma=1$.
We further multiply the squared Euclidean distance $d(x,y)$ in the loss function with a factor of $0.5/\sigma^2$ with $\sigma=0.5$ to control the spread of the embeddings.

\subsubsection{CWMs}
\label{apdx_subsubsec:cwms}
\paragraph{Confounder Estimator}
We choose the confounder estimator $f_\phi$ as a  GRU~\cite{conf/emnlp/ChoMGBBSB14} with 2 layers and a hidden state of dimension 32.

All other settings are the same as the conventional world models in Appendix~\ref{apdx_subsubsec:wms} and we concatenate the abstract state with the estimated confounder to feed into the transition model.

\subsubsection{CRM-CWMs}
\textbf{Propensity Score Estimation}
We estimate the propensity score as a Gaussian Distribution density function,
i.e.,
$
f(x)=\frac{1}{\sigma \sqrt{2 \pi}} e^{-\frac{(x-\mu)^{2}}{2 \sigma^{2}}},
$
where $x$ are all the observations from the training set.
We use~\Eqref{eq:unbias_function} as our training objective.
All other settings are the same as the CWMs in Appendix~\ref{apdx_subsubsec:cwms}.

\subsection{Dream Usability Experiments}
\label{appdx:hyperparam_phyre}
We train all models on an experience buffer obtained by running a random policy on the respective environment.
We choose $70$ tasks as the training set and $30$ tasks as the test set.
At test time, all agents (except the random agent) rank the same set of $200$ actions on each task and propose the highest-scoring actions for that task as solution attempts.
In the PHYRE environment,
the observation is a $256 \times 256$ image with one of $7$ colors at each pixel,
which encodes properties of each body and the goal.
We map this observation into a $7$-channel image for input to the CNN;
each colored pixel in the image yields a $7$D one-hot vector.

\subsubsection{DQN}
\label{apdx_subsubsec:phyre_dqn}
The DQN agent comprises three parts:
\paragraph{Action Encoder}
The action encoder transforms the 3D action $(a_x, a_y, a_r)$ as described in~\Secref{subsec:exp_phyre} using the following structure:
\begin{enumerate}
\item fully connected. 256 ReLU. 
\item fully connected. 128.
\end{enumerate}
\paragraph{Observation Encoder}
The observation encoder transforms the observation image into a hidden representation using a CNN with the following structure:
\begin{enumerate}
\item $1 \times 1$ conv. 3 ReLU. padding 0. stride 1. BatchNorm
\item $7 \times 7$ conv. $64$ ReLU. padding 3. stride 4. BatchNorm
\item $5 \times 5$ conv. $64$ ReLU. padding 2. stride 2. BatchNorm
\item $5 \times 5$ conv. $64$ ReLU. padding 2. stride 2. BatchNorm
\item $5 \times 5$ conv. $64$ ReLU. padding 2. stride 2. BatchNorm
\item $5 \times 5$ conv. $128$ ReLU. padding 2. stride 2. BatchNorm
\item $5 \times 5$ conv. $128$ ReLU. padding 2. stride 2. BatchNorm
\item $5 \times 5$ conv. $128$ ReLU. padding 2. stride 2. BatchNorm
\end{enumerate}
\paragraph{Fusion Module}
The fusion module combines the action and observation representations and makes a reward prediction.

The DQN agent network is trained end-to-end using Adam optimizer by minimizing the cross-entropy between the soft prediction and the observed reward.
The learning rate for DQN is set as $3e-4$.

\subsubsection{WMs}
\label{apdx_subsubsec:usability_wms}
\paragraph{Object Extractor}
We choose the object extractor as the following:
\begin{enumerate}
\item $1 \times 1$ conv. 3 ReLU. padding 0. stride 1. BatchNorm
\item $7 \times 7$ conv. $64$ ReLU. padding 3. stride 4. BatchNorm
\item $5 \times 5$ conv. $64$ ReLU. padding 2. stride 2. BatchNorm
\item $5 \times 5$ conv. $64$ ReLU. padding 2. stride 2. BatchNorm
\item $5 \times 5$ conv. $64$ ReLU. padding 2. stride 2. BatchNorm
\item $5 \times 5$ conv. $K$ ReLU. padding 2. stride 2. BatchNorm
\end{enumerate}
$K$ is a hyperparameter and we typically set it with heuristic as the number of objects in the task,
i.e., $K=3$ for the task in~\Figref{subfig:success_phyre}.

\paragraph{Object Encoder, Transition Function, and the Loss Function}
All modules are the same as the setting in Appendix~\ref{apdx_subsubsec:cwms}.
\paragraph{State Status Classifier}
We train a separate classifier to predict whether the task is solved given an abstract state value.
The classifier is of the following architecture:
\begin{enumerate}
\item fully connected. 256 ReLU. 
\item fully connected. 128 ReLU.
\item fully connected. 2.
\end{enumerate}
The classifier network is trained by minimizing the cross-entropy between the soft prediction and the label (state solved or not).

All models are trained for $200$ epochs using the Adam optimizer with a learning rate of $5e-4$ and a batch size of $16$ (due to the relatively large size of each training tuple).
All experiments were completed on a single NVIDIA GeForce GTX 1080 Ti GPU in under 6 hours.

\subsubsection{CWMs}
\label{apdx_subsubsec:usability_cwms}
\paragraph{Confounder Estimator}
We choose the confounder estimator as the same setting as described in Appendix~\ref{apdx_subsubsec:cwms}.

All other settings are the same as the conventional world models in Appendix~\ref{apdx_subsubsec:usability_wms} and we concatenate the abstract state with the estimated confounder to feed into the transition model.

\section{Additional Experiment Results}
\label{appdx:exp_results}

\subsection{Qualitative Examples of Dream Quality Experiments}
We present additional \emph{dream} (counterfactual prediction) trajectories for different environments in~\Figref{appdxfig:traj_b2} through~\Figref{appdxfig:traj_b6}.
\begin{figure*}[t!]
\centering
\begin{subfigure}{0.32\textwidth}
\centering
\includegraphics[width=0.9\textwidth]{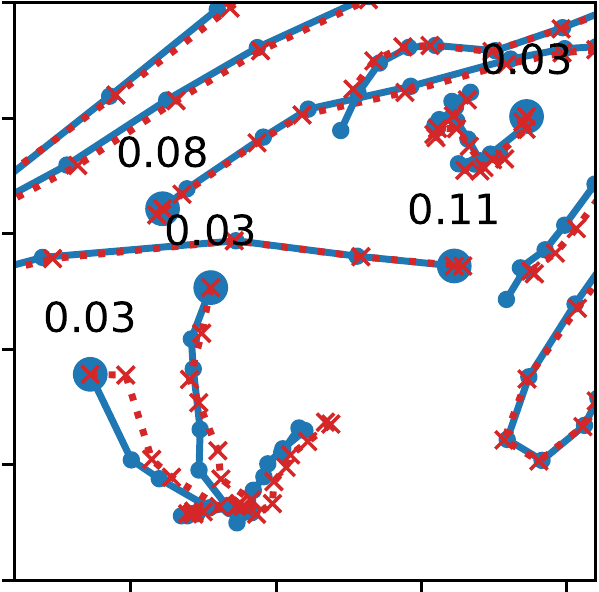}
\caption{WMs: $0.056${\color{lightgrey}\tiny$\pm0.037$}}
\label{appdxsubfig:traj_wm_b2}
\end{subfigure}
\begin{subfigure}{0.32\textwidth}
\centering
\includegraphics[width=0.9\textwidth]{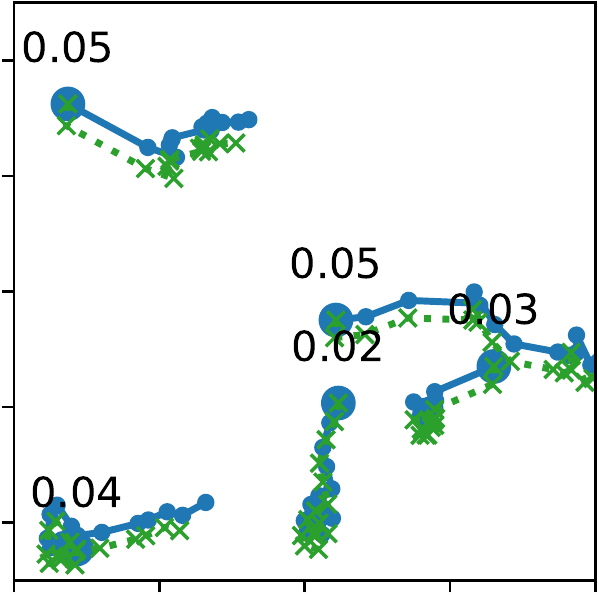}
\caption{CWMs: $0.038${\color{lightgrey}\tiny$\pm0.013$}}
\label{appdxsubfig:traj_cwm_b2}
\end{subfigure}
\begin{subfigure}{0.32\textwidth}
\centering
\includegraphics[width=0.9\textwidth]{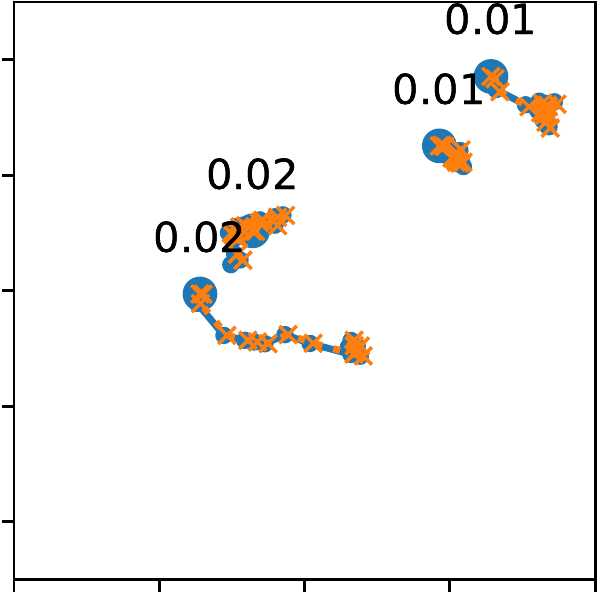}
\caption{CRM-CWMs: $0.015${\color{lightgrey}\tiny$\pm0.006$}}
\label{appdxsubfig:traj_crm_b2}
\end{subfigure}
\caption{
The \emph{dream} (counterfactual prediction) trajectories on (a) WMs (cross with red dotted lines), (b) CWMs (cross with green dotted lines), and (c) CRM-CWMs (cross with orange dotted lines) along with the ground truth (circle with blue solid lines) in the latent space of the 2-ball CoPhy environment.
Trajectories are projected to two dimensions via PCA and presented with the same scale.
The MSE between the predicted state values and the ground truth are presented for each episode.
}
\label{appdxfig:traj_b2}
\end{figure*}

\begin{figure*}[t!]
\centering
\begin{subfigure}{0.32\textwidth}
\centering
\includegraphics[width=0.9\textwidth]{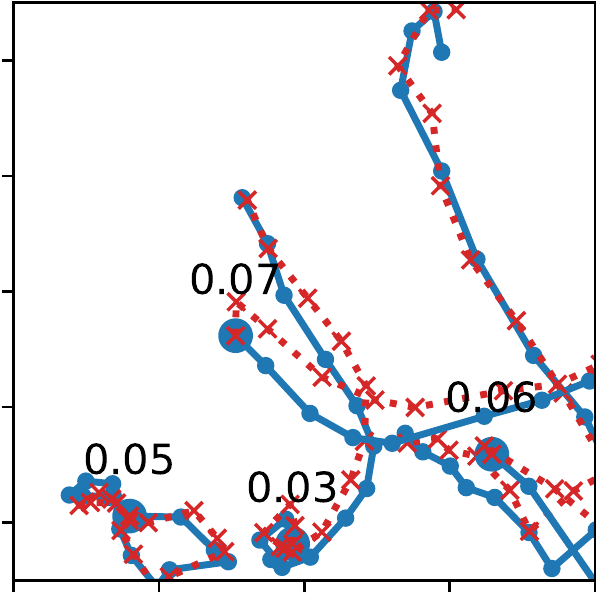}
\caption{WMs: $0.053${\color{lightgrey}\tiny$\pm0.017$}}
\label{appdxsubfig:traj_wm_b3}
\end{subfigure}
\begin{subfigure}{0.32\textwidth}
\centering
\includegraphics[width=0.9\textwidth]{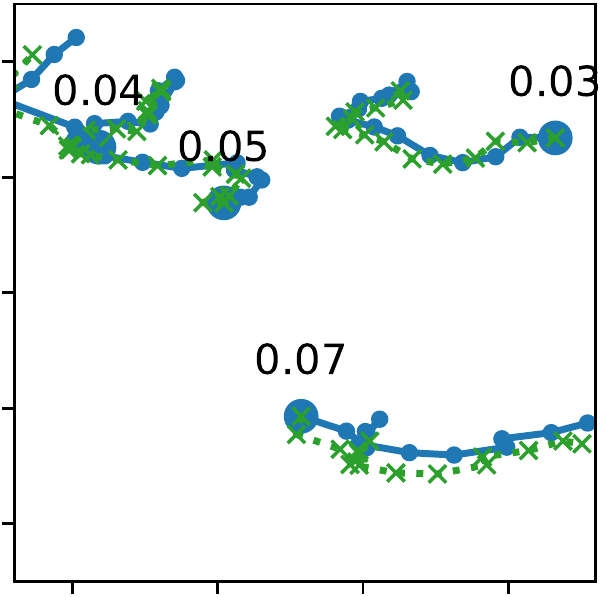}
\caption{CWMs: $0.048${\color{lightgrey}\tiny$\pm0.017$}}
\label{appdxsubfig:traj_cwm_b3}
\end{subfigure}
\begin{subfigure}{0.32\textwidth}
\centering
\includegraphics[width=0.9\textwidth]{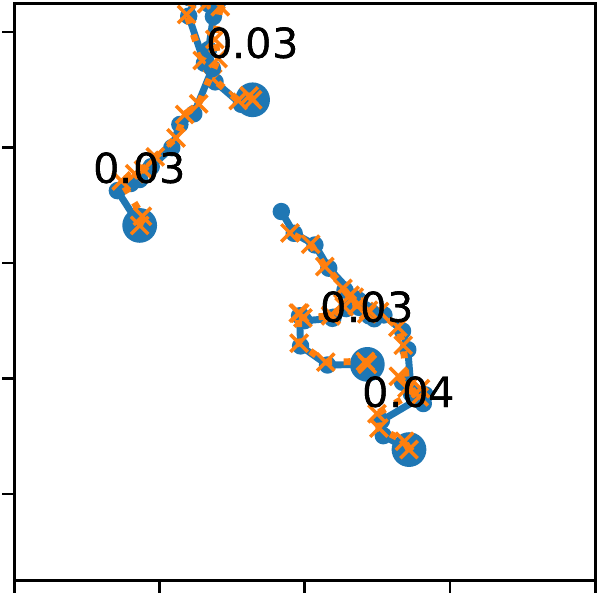}
\caption{CRM-CWMs: $0.033${\color{lightgrey}\tiny$\pm0.005$}}
\label{appdxsubfig:traj_crm_b3}
\end{subfigure}
\caption{
The \emph{dream} (counterfactual prediction) trajectories on (a) WMs (cross with red dotted lines), (b) CWMs (cross with green dotted lines), and (c) CRM-CWMs (cross with orange dotted lines) along with the ground truth (circle with blue solid lines) in the latent space of the 3-ball CoPhy environment.
Trajectories are projected to two dimensions via PCA and presented with the same scale.
The MSE between the predicted state values and the ground truth are presented for each episode.
}
\label{appdxfig:traj_b3}
\end{figure*}

\begin{figure*}[t!]
\centering
\begin{subfigure}{0.32\textwidth}
\centering
\includegraphics[width=0.9\textwidth]{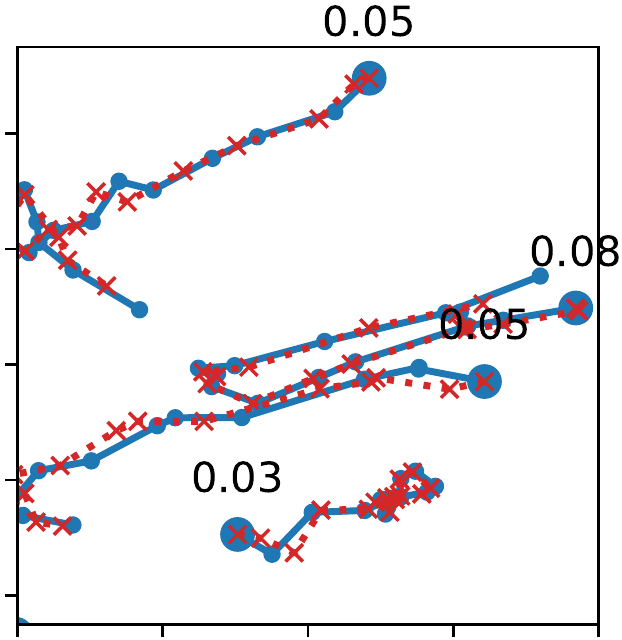}
\caption{WMs: $0.053${\color{lightgrey}\tiny$\pm0.021$}}
\label{appdxsubfig:traj_wm_b4}
\end{subfigure}
\begin{subfigure}{0.32\textwidth}
\centering
\includegraphics[width=0.86\textwidth]{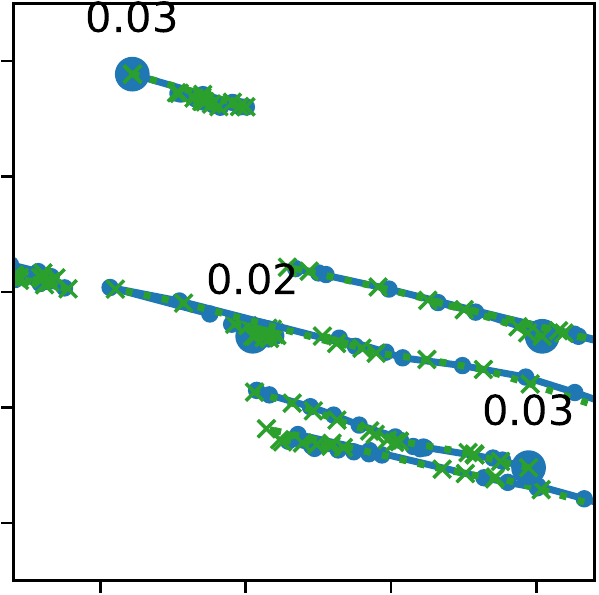}
\caption{CWMs: $0.027${\color{lightgrey}\tiny$\pm0.006$}}
\label{appdxsubfig:traj_cwm_b4}
\end{subfigure}
\begin{subfigure}{0.32\textwidth}
\centering
\includegraphics[width=0.9\textwidth]{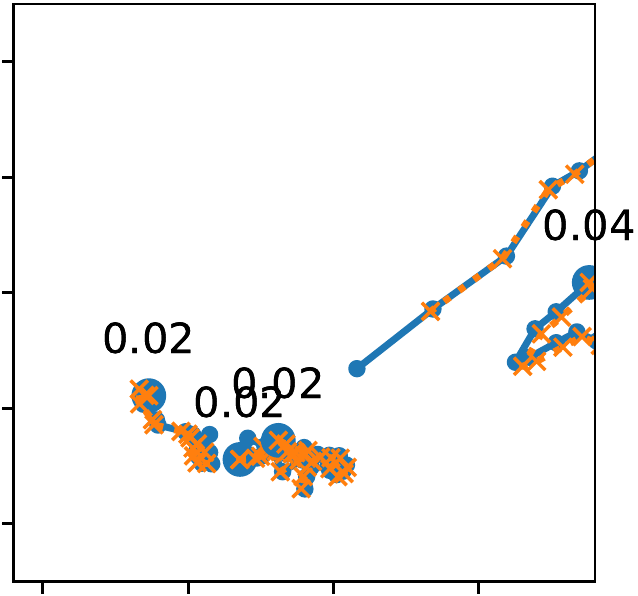}
\caption{CRM-CWMs: $0.025${\color{lightgrey}\tiny$\pm0.01$}}
\label{appdxsubfig:traj_crm_b4}
\end{subfigure}
\caption{
The \emph{dream} (counterfactual prediction) trajectories on (a) WMs (cross with red dotted lines), (b) CWMs (cross with green dotted lines), and (c) CRM-CWMs (cross with orange dotted lines) along with the ground truth (circle with blue solid lines) in the latent space of the 4-ball CoPhy environment.
Trajectories are projected to two dimensions via PCA and presented with the same scale.
The MSE between the predicted state values and the ground truth are presented for each episode.
}
\label{appdxfig:traj_b4}
\end{figure*}

\begin{figure*}[t!]
\centering
\begin{subfigure}{0.32\textwidth}
\centering
\includegraphics[width=0.9\textwidth]{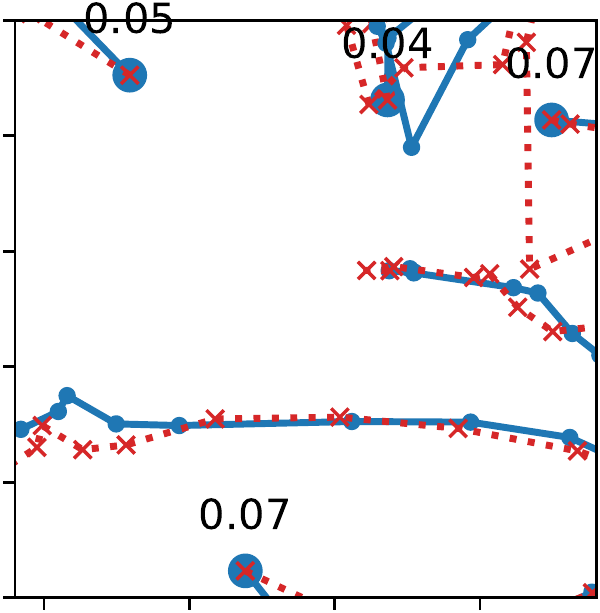}
\caption{WMs: $0.058${\color{lightgrey}\tiny$\pm0.015$}}
\label{appdxsubfig:traj_wm_b5}
\end{subfigure}
\begin{subfigure}{0.32\textwidth}
\centering
\includegraphics[width=0.95\textwidth]{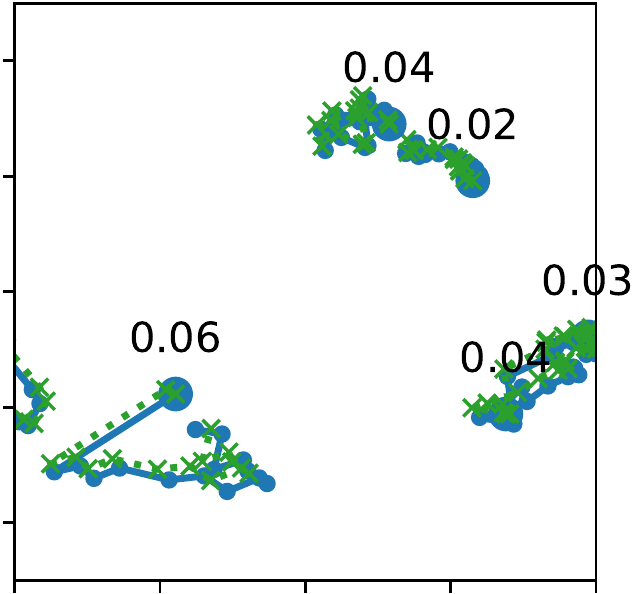}
\caption{CWMs: $0.038${\color{lightgrey}\tiny$\pm0.015$}}
\label{appdxsubfig:traj_cwm_b5}
\end{subfigure}
\begin{subfigure}{0.32\textwidth}
\centering
\includegraphics[width=0.9\textwidth]{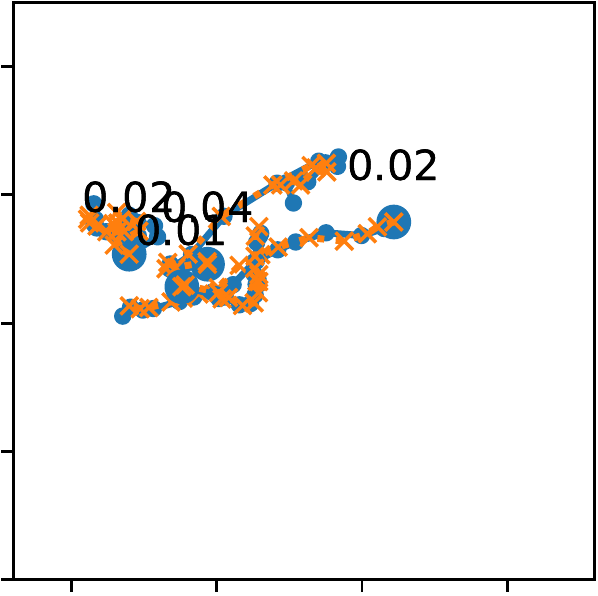}
\caption{CRM-CWMs: $0.023${\color{lightgrey}\tiny$\pm0.013$}}
\label{appdxsubfig:traj_crm_b5}
\end{subfigure}
\caption{
The \emph{dream} (counterfactual prediction) trajectories on (a) WMs (cross with red dotted lines), (b) CWMs (cross with green dotted lines), and (c) CRM-CWMs (cross with orange dotted lines) along with the ground truth (circle with blue solid lines) in the latent space of the 5-ball CoPhy environment.
Trajectories are projected to two dimensions via PCA and presented with the same scale.
The MSE between the predicted state values and the ground truth are presented for each episode.
}
\label{appdxfig:traj_b5}
\end{figure*}

\begin{figure*}[t!]
\centering
\begin{subfigure}{0.32\textwidth}
\centering
\includegraphics[width=0.92\textwidth]{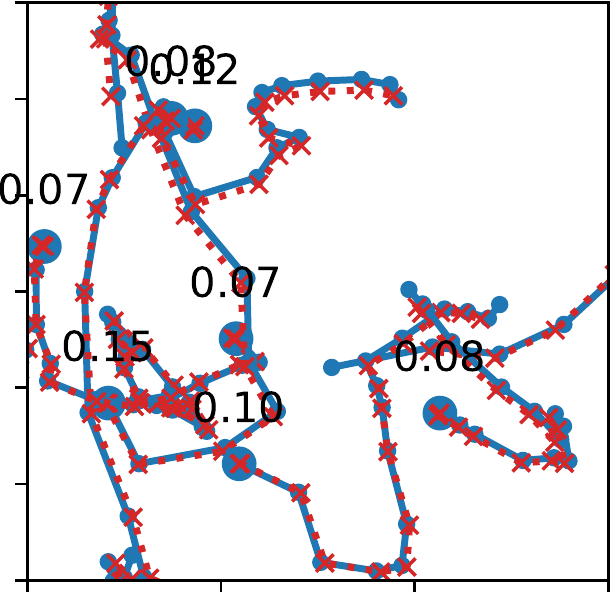}
\caption{WMs: $0.096${\color{lightgrey}\tiny$\pm0.030$}}
\label{appdxsubfig:traj_wm_b6}
\end{subfigure}
\begin{subfigure}{0.32\textwidth}
\centering
\includegraphics[width=0.9\textwidth]{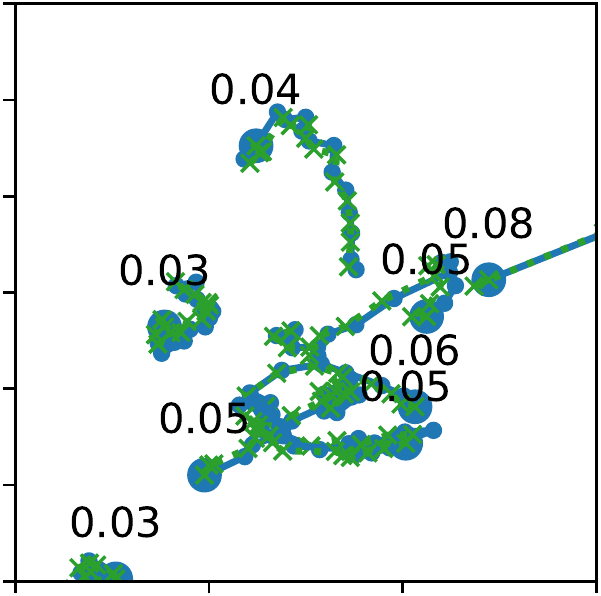}
\caption{CWMs: $0.049${\color{lightgrey}\tiny$\pm0.016$}}
\label{appdxsubfig:traj_cwm_b6}
\end{subfigure}
\begin{subfigure}{0.32\textwidth}
\centering
\includegraphics[width=0.9\textwidth]{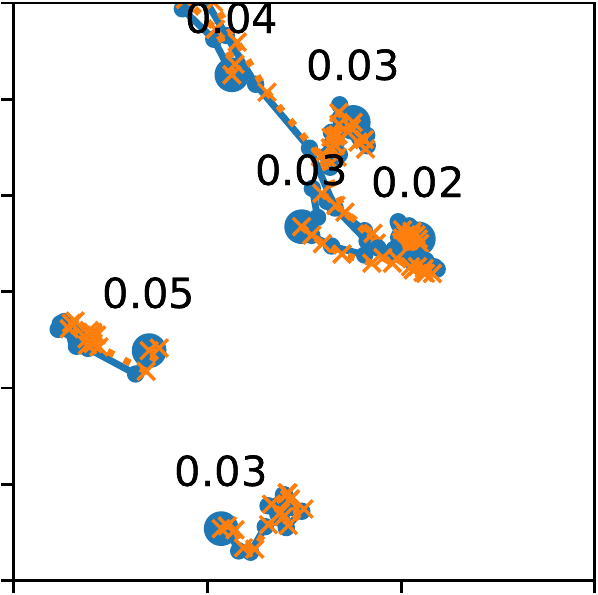}
\caption{CRM-CWMs: $0.033${\color{lightgrey}\tiny$\pm0.010$}}
\label{appdxsubfig:traj_crm_b6}
\end{subfigure}
\caption{
The \emph{dream} (counterfactual prediction) trajectories on (a) WMs (cross with red dotted lines), (b) CWMs (cross with green dotted lines), and (c) CRM-CWMs (cross with orange dotted lines) along with the ground truth (circle with blue solid lines) in the latent space of the 6-ball CoPhy environment.
Trajectories are projected to two dimensions via PCA and presented with the same scale.
The MSE between the predicted state values and the ground truth are presented for each episode.
}
\label{appdxfig:traj_b6}
\end{figure*}

\newpage
\subsection{Ablation Study on the Choice of $K$ in Dream Quality Experiments}
\begin{figure*}[t!]
\centering
\begin{subfigure}{0.5\textwidth}
\centering
\includegraphics[width=\textwidth]{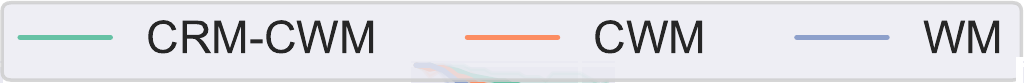}
\label{fig:ballscflegengd}
\end{subfigure}
\begin{subfigure}{\textwidth}
\centering
\includegraphics[width=\textwidth]{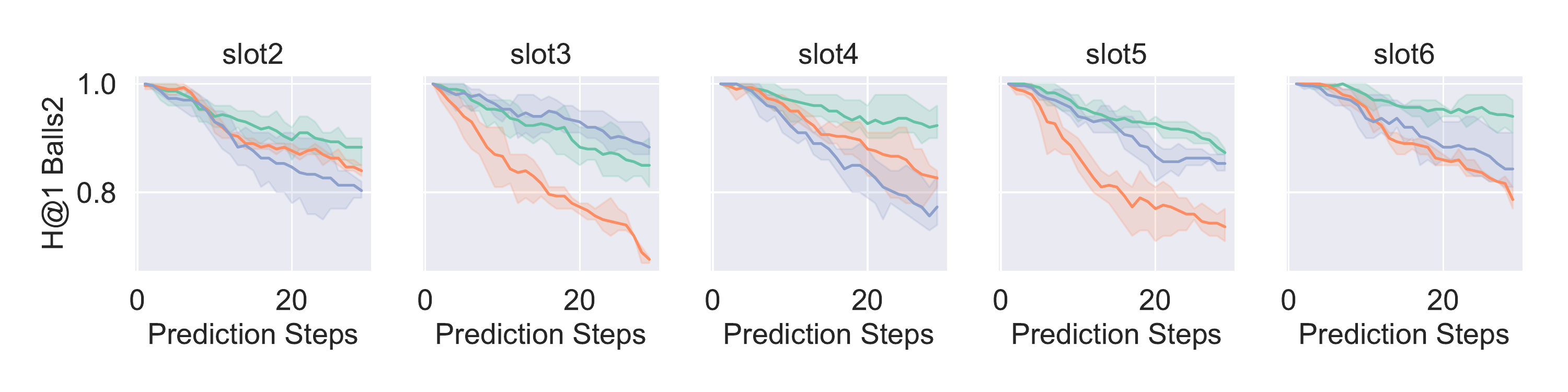}
\label{fig:ballscf2}
\end{subfigure}
\vskip -0.3in
\begin{subfigure}{\textwidth}
\centering
\includegraphics[width=\textwidth]{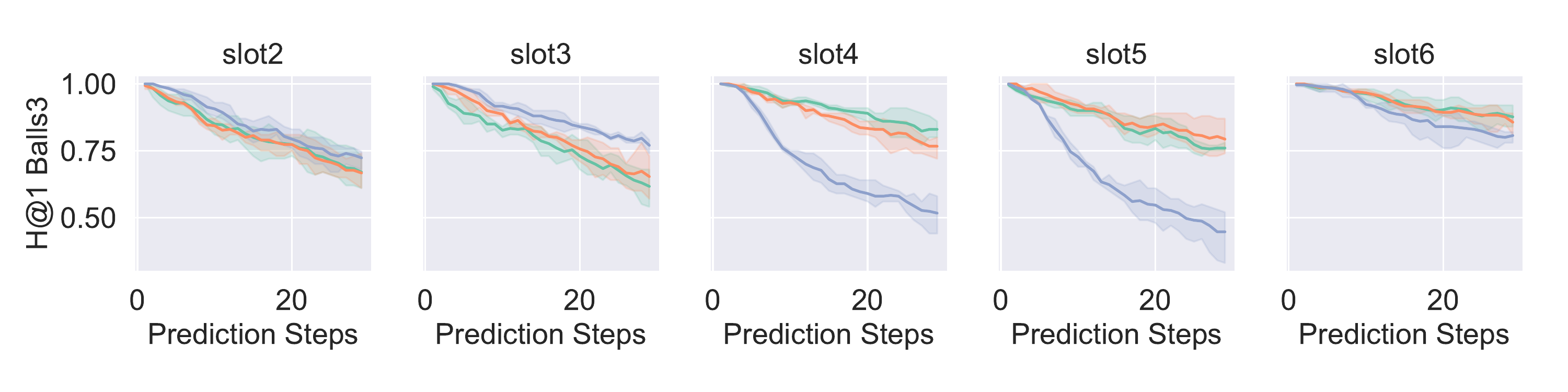}
\label{fig:ballscf4}
\end{subfigure}
\vskip -0.3in
\begin{subfigure}{\textwidth}
\centering
\includegraphics[width=\textwidth]{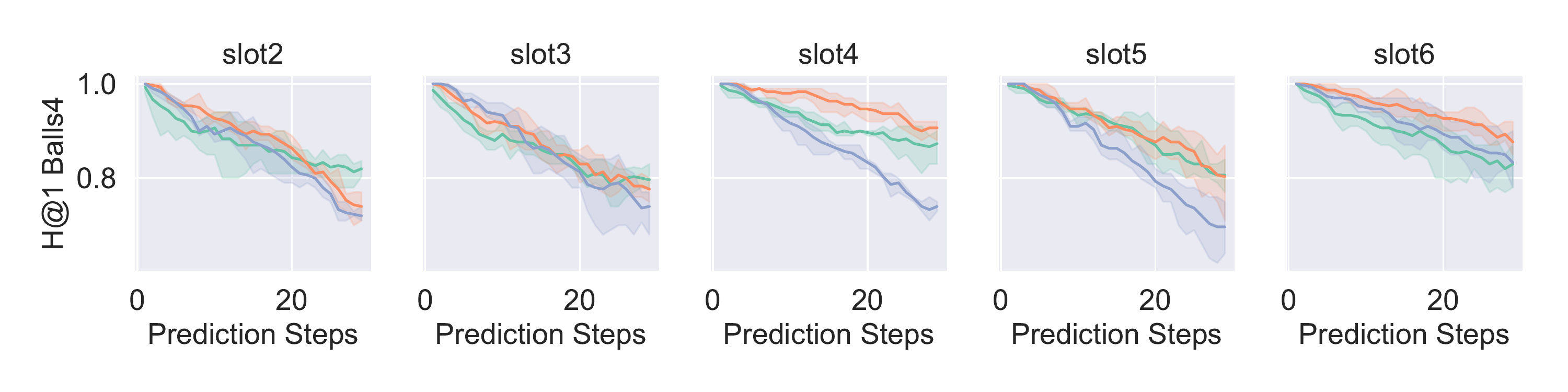}
\label{fig:ballscf6}
\end{subfigure}
\vskip -0.3in
\begin{subfigure}{\textwidth}
\centering
\includegraphics[width=\textwidth]{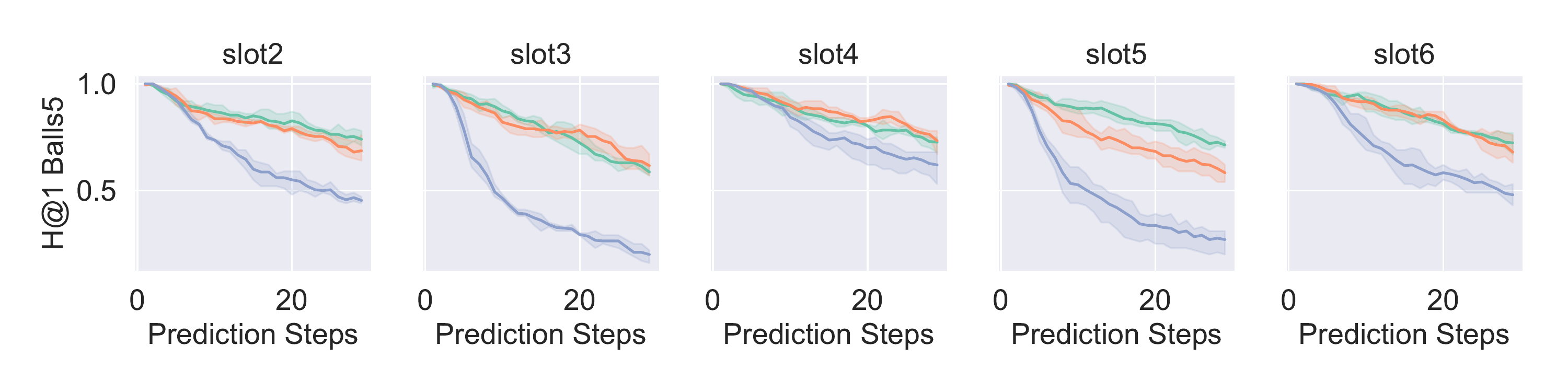}
\label{fig:ballscf6}
\end{subfigure}
\vskip -0.3in
\begin{subfigure}{\textwidth}
\centering
\includegraphics[width=\textwidth]{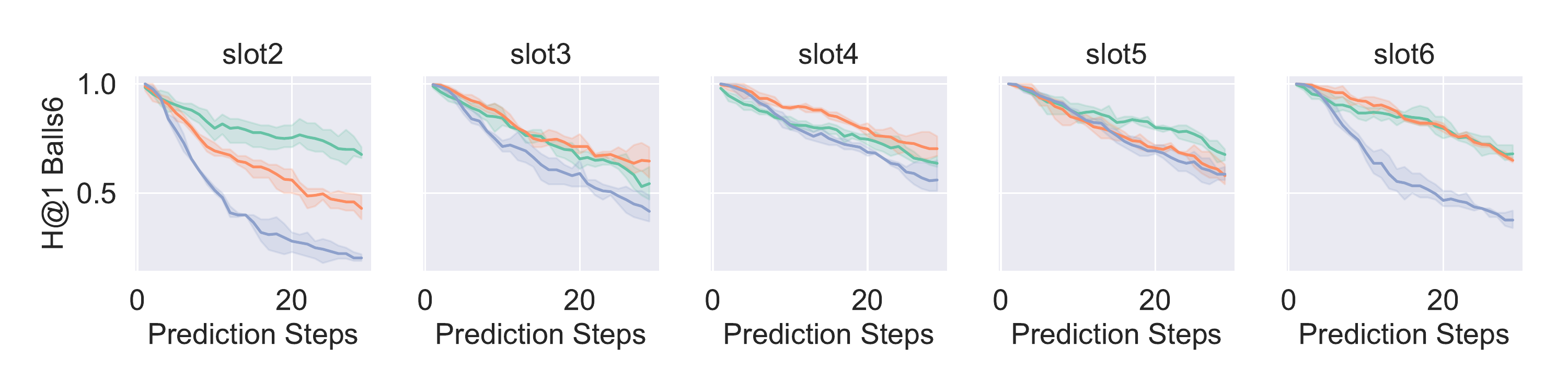}
\label{fig:ballscf6}
\end{subfigure}
\vskip -0.3in
\caption{(H@1)
Ranking results for multi-step prediction in latent space with different slot (K) variations.
}
\label{appdxfig:exp_h1}
\end{figure*}

\begin{figure*}[t!]
\centering
\begin{subfigure}{0.5\textwidth}
\centering
\includegraphics[width=\textwidth]{figure/separate/legend.pdf}
\label{fig:ballscflegengd}
\end{subfigure}
\begin{subfigure}{\textwidth}
\centering
\includegraphics[width=\textwidth]{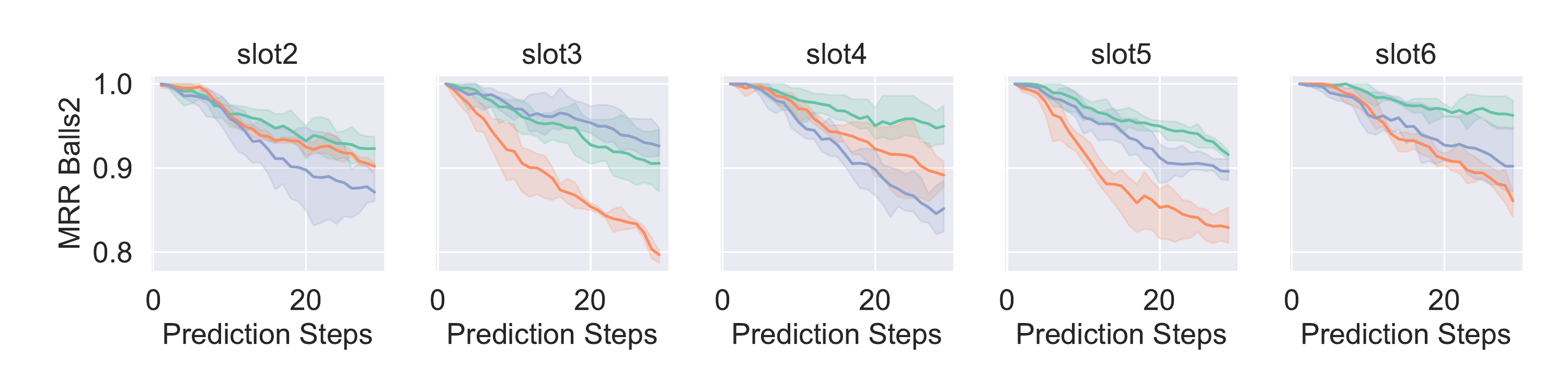}
\label{fig:ballscf2}
\end{subfigure}
\vskip -0.3in
\begin{subfigure}{\textwidth}
\centering
\includegraphics[width=\textwidth]{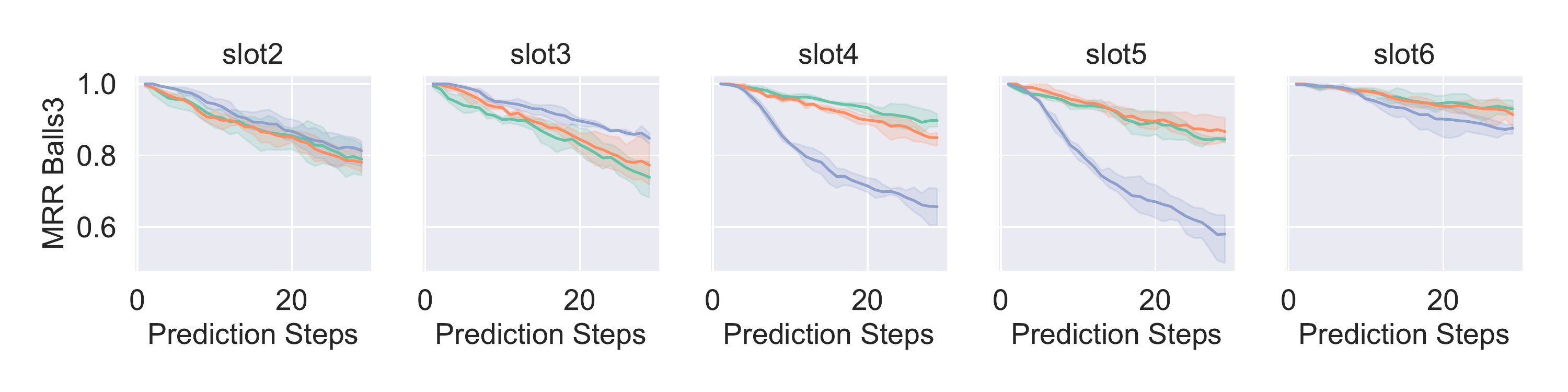}
\label{fig:ballscf4}
\end{subfigure}
\vskip -0.3in
\begin{subfigure}{\textwidth}
\centering
\includegraphics[width=\textwidth]{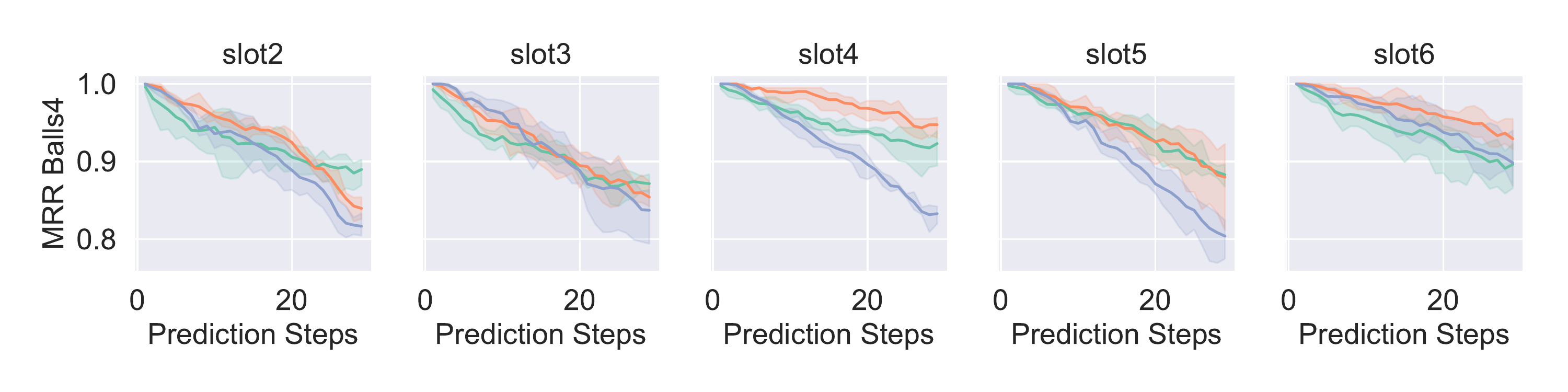}
\label{fig:ballscf6}
\end{subfigure}
\vskip -0.3in
\begin{subfigure}{\textwidth}
\centering
\includegraphics[width=\textwidth]{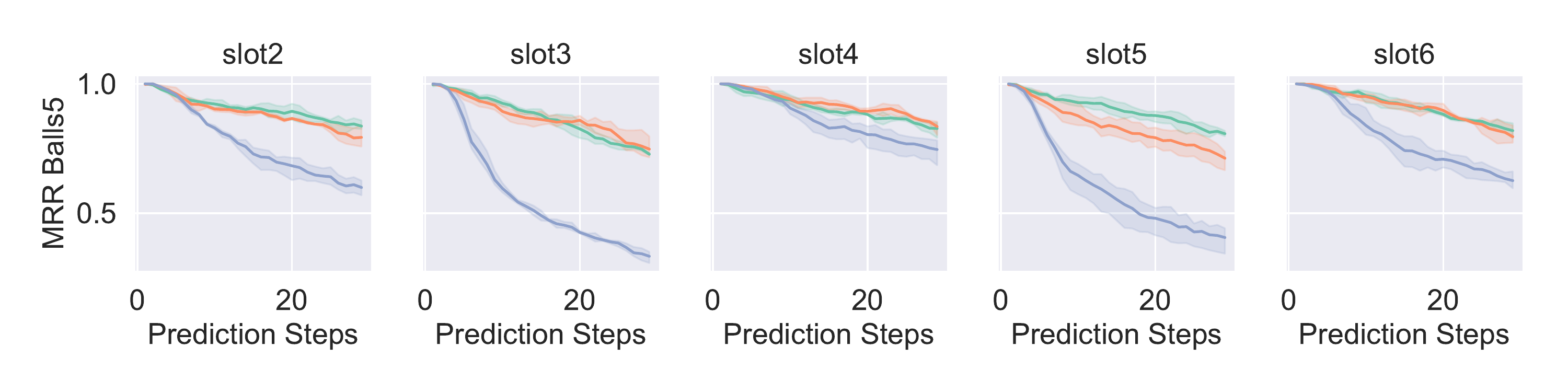}
\label{fig:ballscf6}
\end{subfigure}
\vskip -0.3in
\begin{subfigure}{\textwidth}
\centering
\includegraphics[width=\textwidth]{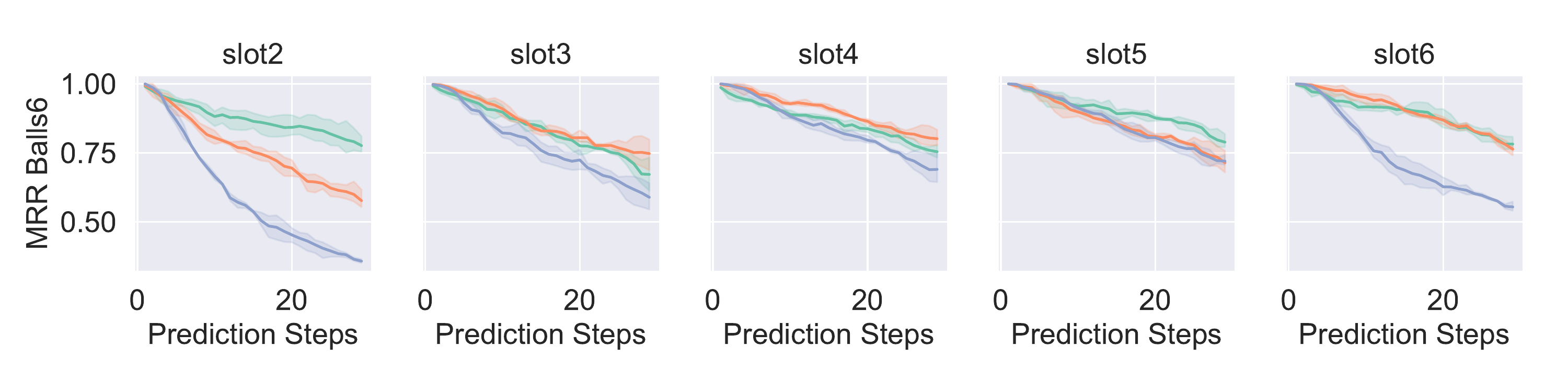}
\label{fig:ballscf6}
\end{subfigure}
\vskip -0.3in
\caption{(MRR)
Ranking results for multi-step prediction in latent space with different slot (K) variations. Our model (Causal SWM) consistently achieves the best result.
}
\label{appdxfig:exp_mrr}
\end{figure*}
From experimental results, we also discover some insights about parameter tuning. 
With the increasing number of balls in the environment,
all methods have been struggling to make good long term predictions. 
Although results in~\Figref{fig:exp_metrics_h1} and~\Figref{fig:exp_metrics_mrr} of the main paper are presented using the best configuration of number of object slots $K$ for each model,
results in~\Figref{appdxfig:exp_h1} and~\Figref{appdxfig:exp_mrr} show that our proposed CWMs and CRM-CWMs benefit from using a larger number of slots in all environments.
This finding suggests that using a large number of slots $K$ might be a good choice when dealing with a new environment without any prior information.

\end{document}